\documentclass[sigconf]{aamas} 

\usepackage{balance} 
\usepackage{enumitem}
\usepackage{bbm,bm} 
\usepackage{soul}
\usepackage{url}
\usepackage[utf8]{inputenc}
\usepackage{caption}
\usepackage{subfigure}
\usepackage{graphicx}
\usepackage{float}
\usepackage{booktabs}
\usepackage{tabularx}
\usepackage{algorithm}
\usepackage{algorithmic}
\usepackage{color}
\usepackage{multirow}
\usepackage{appendix}
\usepackage{amsmath,amsthm,mathtools}
\usepackage{wrapfig}

\usepackage{amsmath,amsfonts,bm}

\def\eqref#1{equation~\ref{#1}}

\def\1{\bm{1}}

\DeclareMathAlphabet{\mathsfit}{\encodingdefault}{\sfdefault}{m}{sl}
\SetMathAlphabet{\mathsfit}{bold}{\encodingdefault}{\sfdefault}{bx}{n}

\def\gN{{\mathcal{N}}}

\DeclareMathOperator*{\argmax}{arg\,max}
\DeclareMathOperator*{\argmin}{arg\,min}

\newtheorem{lemma}{Lemma}
\newtheorem{proposition}{Proposition}

\newcommand{\fig}[1]{Fig.~\ref{#1}}
\newcommand{\eq}[1]{Eq.~(\ref{#1})}
\newcommand{\tb}[1]{Tab.~\ref{#1}}
\newcommand{\se}[1]{Section~\ref{#1}}
\newcommand{\ap}[1]{Appendix~\ref{#1}}

\newcommand{\prop}[1]{Proposition~\ref{#1}}
\newcommand{\alg}[1]{Algo.~\ref{#1}}

\DeclareMathOperator{\rl}{RL}
\DeclareMathOperator{\irl}{IRL} 

\newcommand*{\dif}{\mathop{}\!\mathrm{d}}

\newcommand{\bbE}{\ensuremath{\mathbb{E}}} 
\newcommand{\bbR}{\ensuremath{\mathbb{R}}} 
\newcommand{\caA}{\ensuremath{\mathcal{A}}} 
\newcommand{\caS}{\ensuremath{\mathcal{S}}} 
\newcommand{\caM}{\ensuremath{\mathcal{M}}} 
\newcommand{\caD}{\ensuremath{\mathcal{D}}} 
\newcommand{\caL}{\ensuremath{\mathcal{L}}} 
\newcommand{\kld}{\text{D}_{\text{KL}}} 

\newcommand{\piE}{{\pi_E}}
\newcommand{\hr}{\hat{r}}

\setcopyright{ifaamas}
\acmConference[AAMAS '21]{Proc.\@ of the 20th International Conference on Autonomous Agents and Multiagent Systems (AAMAS 2021)}{May 3--7, 2021}{Online}{U.~Endriss, A.~Now\'{e}, F.~Dignum, A.~Lomuscio (eds.)}
\copyrightyear{2021}
\acmYear{2021}
\acmDOI{}
\acmPrice{}
\acmISBN{}

\acmSubmissionID{126}

\title{Energy-Based Imitation Learning}

\author{Minghuan Liu, Tairan He, Minkai Xu, Weinan Zhang$^{\dagger}$}
\affiliation{%
  \institution{Shanghai Jiaotong University}
  \city{Shanghai, China}
}
\email{{minghuanliu, whynot, wnzhang}@sjtu.edu.cn, mkxu@apex.sjtu.edu.cn}

\thanks{$\dagger$ Corresponding author}

\begin{abstract}
We tackle a common scenario in imitation learning (IL), where agents try to recover the optimal policy from expert demonstrations without further access to the expert or environment reward signals. Except the simple Behavior Cloning (BC) that adopts supervised learning followed by the problem of compounding error, previous solutions like inverse reinforcement learning (IRL) and recent generative adversarial methods involve a bi-level or alternating optimization for updating the reward function and the policy, suffering from high computational cost and training instability. Inspired by recent progress in energy-based model (EBM), in this paper, we propose a simplified IL framework named Energy-Based Imitation Learning (EBIL). Instead of updating the reward and policy iteratively, EBIL breaks out of the traditional IRL paradigm by a simple and flexible two-stage solution: first estimating the expert energy as the surrogate reward function through score matching, then utilizing such a reward for learning the policy by reinforcement learning algorithms. EBIL combines the idea of both EBM and occupancy measure matching, and via theoretic analysis we reveal that EBIL and Max-Entropy IRL (MaxEnt IRL) approaches are two sides of the same coin, and thus EBIL could be an alternative of adversarial IRL methods. Extensive experiments on qualitative and quantitative evaluations indicate that EBIL is able to recover meaningful and interpretative reward signals while achieving effective and comparable performance against existing algorithms on IL benchmarks.
\end{abstract}

\begin{document}



         
\newcommand{\BibTeX}{\rm B\kern-.05em{\sc i\kern-.025em b}\kern-.08em\TeX}

\pagestyle{fancy}
\fancyhead{}

\maketitle

\section{Introduction}


Imitation learning (IL)~\cite{hussein2017imitation} allows Reinforcement Learning (RL) agents to learn from demonstrations, without any further access to the expert or explicit rewards. Classic solutions for IL such as behavior cloning (BC)~\cite{pomerleau1991efficient} aim to minimize 1-step deviation error along the provided expert trajectories with supervised learning, which requires an extensive collection of expert data and suffers seriously from compounding error caused by covariate shift ~\cite{ross2010efficient,ross2011reduction} due to the long-term trajectories mismatch when we just clone each single-step action.
As another solution, Inverse Reinforcement Learning (IRL)~\cite{ng2000algorithms,abbeel2004apprenticeship,fu2017learning} tries to recover a reward function from the expert and subsequently train an RL policy under that, yet such a bi-level optimization scheme can result in high computational cost. The recent generative adversarial solution ~\cite{ho2016generative,fu2017learning,ghasemipour2019divergence} takes advantage of GAN~\cite{goodfellow2014generative} to minimize the divergence between the agent's occupancy measure and the expert's while it also inherits the training instability of GAN~\cite{brock2018large}.

Analogous to IL, learning statistical models from given data and generating similar samples has been an important topic in the generative model community. Among them, recent energy-based models (EBMs) have gained much attention because of the simplicity and flexibility in likelihood estimation~\cite{du2019implicit,song2019sliced}. 
In this paper, we propose to leverage the advantages of EBMs to solve IL with a novel but simplified framework called Energy-Based Imitation Learning (EBIL), solving IL in a two-step fashion: first estimates an unnormalized probability density (\textit{a.k.a.} energy) of expert's occupancy measure through score matching, then takes the energy to construct a surrogate reward function as a guidance for the agent to learn the desired policy. 
We realize that EBIL is high related to MaxEnt IRL, and in detail analyze their relation, which reveals that these two methods are two sides of the same coin, and MaxEnt IRL can be seen as a special form of EBIL because MaxEnt IRL estimates the energy and the policy alternately. Therefore, we can think of EBIL as an simplified alternative of adversarial IRL methods.

In experiments, we first verify the effectiveness of EBIL and the meaningful reward recovered in a simple one-dimensional environment by visualizing the estimated reward and the induced policy; then we evaluate our algorithm on extensive high-dimensional continuous control benchmarks, contains sub-optimal expert demonstrations and optimal demonstrations, showing that EBIL can achieve comparable and stable performance against previous adversarial IRL methods. We also show the functionality for resolving state-only imitation learning by matching the target state marginal distribution, and provide evaluations and ablation study on the recovered EBM.

The remainder of this paper is organized as follows. In \se{sec:ebil}, we first present the formulation of the energy-based imitation learning (EBIL) framework, which is a general and principled two-stage RL framework that models the expert policy with an unnormalized probability function (\textit{i.e.}, energy function). In \se{sec:relation}, we give a comprehensive discussion about EBIL and classic IRL methods, which are also built upon the energy-based formulation to model the expert trajectories but typically adopt an adversarial (or alternating) training scheme. The discussion allows us to clarify how to avoid the interactive training of IRL and thus leads to our simplified and principled two-stage algorithm. After that, in \se{sec:energy-estimation}, we practically illustrate how to implement the two-stage energy-based algorithm via score matching method. Lastly, in \se{sec:exps}, through comprehensive experiments from synthetic domain to continuous control tasks, we demonstrate the interpretability and effectiveness of EBIL over existing algorithms.
\section{Background}
We consider a Markov Decision Process (MDP) $\caM = \langle \caS, \caA, P, \rho_0, r,\\ \gamma \rangle$, where $\caS$ is the set of states, $\caA$ represents the action space of the agent, $P: \caS \times \caA \times \caS \rightarrow [0, 1]$ is the state transition probability distribution, $\rho_0: \caS\rightarrow[0,1]$ is the distribution of the initial state $s_0$, and $\gamma\in [0,1]$ is the discounted factor. The agent holds its policy $\pi(a|s): \caS \times \caA \rightarrow [0, 1]$ to make decisions and receive rewards defined as $r: \caS \times \caA \rightarrow \mathbb{R}$. For an arbitrary function $f: \langle s, a \rangle \rightarrow \mathbb{R}$, we denote the expectation w.r.t. the policy $\pi$ as $\bbE_\pi[f(s,a)]\triangleq \bbE_{s_0\sim\rho_0, s_t\sim P, a_t\sim \pi} \left [\sum_{t=0}^\infty \gamma^t f(s_t,a_t) \right ]$. The objective of Maximum Entropy Reinforcement Learning (MaxEnt RL) is required to find a stochastic policy that can maximize its reward along with the entropy~\cite{ho2016generative,haarnoja2017reinforcement} as:
\begin{equation}\label{eq:maxent-rl}
    \pi^* = \argmax_{\pi} \bbE_\pi\left [r(s,a)\right ]+\alpha H(\pi)~,
\end{equation}
where $H(\pi) \triangleq \bbE_{\pi}[- \log \pi(a|s)]$ is the $\gamma$-discounted causal entropy~\cite{bloem2014infinite} and $\alpha$ is the temperature hyperparameter. Throughout this work we denote the occupancy measure $\rho_\pi^{s,a}(s,a)$ or $\rho_\pi^{s}(s)$ as the density of occurrence of states or state-action pairs\footnote{It is important to note that the definition of occupancy measure is not equivalent to the definition of a normalized distribution since in RL we have to deal with the discounted factor for the expectation w.r.t. the policy $\bbE_\pi$.}:
\begin{equation}
\begin{aligned}
    \rho_{\pi}^{s,a}(s,a) &= \sum_{t=0}^{\infty}\gamma^t P(s_t=s, a_t=a|\pi)\\
&= \pi(a|s)\sum_{t=0}^{\infty}\gamma^t P(s_t=s|\pi)
= \pi(a|s)\rho_{\pi}^{s}(s)~,
\end{aligned}
\end{equation}
which allows us to write  $\bbE_{\pi}[\boldsymbol{\cdot}]=\sum_{s,a}\rho_{\pi}^{s,a}(s,a)[\boldsymbol{\cdot} ]=\bbE_{(s,a)\sim \rho_{\pi}^{s,a}}[\boldsymbol{\cdot}]$.
For simplicity, we will denote $\rho_\pi^{s,a}$ as $\rho_\pi$ without further explanation, and $\rho_\pi \in \caD\triangleq\{ \rho_\pi: \pi\in\Pi \}$. 

\paragraph{General Imitation Learning}

Imitation learning (IL) \cite{hussein2017imitation} studies the task of Learning from Demonstrations (LfD), which aims to learn a policy from expert demonstrations. The expert demonstrations typically consist of the expert trajectories interacted with environments without any reward signals. General IL objective tries to minimize the policy distance or the occupancy measure distance\footnote{The equivalence of these two objective usually can be easily shown by the one-to-one correspondence of $\pi$ and $\rho$ and the convexity of $\ell$.}:
\begin{equation}\label{eq:il}
\pi^* = \argmin_\pi \mathbb{E}_{s \sim \rho_{\pi}^{s}}\left [\ell \left (\piE(\cdot | s), \pi(\cdot | s)\right )\right ] = \ell \left (\rho_{\piE}, \rho_{\pi}\right )~,
\end{equation}
where $\ell$ denotes some distance metric.
However, as we do not ask the expert agent for further demonstrations, it is always hard to optimize \eq{eq:il} with only expert trajectories accessible. Thus, Behavior Cloning (BC)~\cite{pomerleau1991efficient} provides a straightforward method by learning the policy in a supervised way, where the objective is represented as a Maximum Likelihood Estimation (MLE):
\begin{equation}
\hat{\pi}^* = \argmin_\pi \mathbb{E}_{s \sim \rho_{\pi_E}^s}\left [\ell \left (\piE(\cdot | s), \pi(\cdot | s)\right )\right ]~,
\end{equation}
which suffers from covariate shift problem \cite{ho2016generative} for the i.i.d. state assumption.

\paragraph{Maximum-Entropy Inverse Reinforcement Learning}
Another branch of methods is Inverse Reinforcement Learning (IRL) \cite{ng2000algorithms} that tries to recover the reward function $r^*$ in the environment, with underlying assumption that expert policy is optimal under some reward function $r^*$. A typically maximum-entropy (MaxEnt) solution~\cite{ziebart2008maximum} models the expert trajectories with a Boltzmann distribution\footnote{Note that \eq{eq:irl} is formulated under the deterministic MDP setting. A general form for stochastic MDP is derived in \cite{ziebart2008maximum,ziebart2010modeling} yet owns similar analysis: the probability of a trajectory is decomposed as the product of conditional probabilities of the states $s_t$, which can factor out of all likelihood ratios since they are not affected by the reward function.}:
\begin{equation}\label{eq:irl}
p_\theta(\tau)=\frac{1}{Z}\exp{(r_\theta(\tau))}~,
\end{equation}
where $\tau=\{s_1,a_1,\dots, s_t, a_t \}$ is a trajectory with horizon $t$, $r_\theta(\tau)=\sum_t \gamma^t r_{\theta}(s_t,a_t)$ is a parameterized reward function, and the partition function $Z\triangleq \int\exp{(r_\theta(\tau))} \dif\tau$ is the integral over all trajectories. IRL algorithms suffer from the computational challenge in estimating the partition function $Z$ and the bi-level optimization by alternating between updating the cost function and an optimal policy w.r.t the current cost function with RL.
 
Derived from MaxEntIRL, Generative Adversarial Imitation Learning (GAIL) \cite{ho2016generative} shows that the objective of MaxEntIRL is a dual problem of occupancy measure matching, and thus can be solved through generative models such as GAN. Specifically, it shows that the policy learned by RL on the reward recovered by IRL can be characterized by
\begin{equation}\label{eq:rl-irl}
    \begin{aligned}
    \rl \circ \irl_\psi(\pi_E) = \argmin_\pi - H(\pi) + \psi^*(\rho_\pi - \rho_{\pi_E})~,
    \end{aligned}
\end{equation}
where $\psi$ is the regularizer, and $f^*:\bbR^{\caS\times\caA}\rightarrow\overline{\bbR}$ is the convex conjugate for an arbitrary function $f:\bbR^{\caS\times\caA}\rightarrow \overline{\bbR}$ given by $f^*(x)=\sup_{y\in\bbR^{\caS\times\caA}}x^Ty - f(y)$. 
\eq{eq:rl-irl} shows that various settings of $\psi$ can be seen as a distance metric leading to various solutions of IL. For example, in \cite{ho2016generative} they choose a special form of $\psi$ so that the second term becomes minimizing the JS divergence and \cite{ghasemipour2019divergence} further shows the second term can be any $f$-divergence measure of $\rho$ and $\rho_E$.

\paragraph{Energy-Based Models}
\label{subsec:ebm}
The energy term is originally borrowed from statistical physics, where it is employed for describing the distribution of the atoms or molecules. Low-energy states correspond to the high probability of occurring, i.e., the local minima of this function are usually related to the stable stationary states. Such a property is appropriate for modeling the density of a data distribution, where high-density should be assigned with lower energy, and higher energy otherwise. Formally, for a random variable $X\sim p(x)$, an Energy-Based Model (EBM) \cite{lecun2006tutorial} builds the density of data by estimating the energy function $E(x)$ with sample $x$ as a Boltzmann distribution:
\begin{equation}
    \begin{aligned}
        p(x) = \frac{1}{Z}\exp(-E(x))~,
    \end{aligned}
\end{equation}
where $Z=\int\exp(-E(x))\dif x$ is the partition function, which is normally intractable to compute exactly for
high-dimensional $x$. As shown in \fig{fig:energy-example}, the energy function $E$ can be seen as the unnormalized log-density of data which is always optimized to maximize the likelihood of the data. Therefore, we can model the expert demonstrations (state-action pairs) with the energy function, where low energy corresponds to the state-action pairs that the expert mostly perform. Typically, the estimation of the partition function $Z$ is computationally expensive, which requires sampling from the Boltzmann distribution $p(x)$ within the inner loop of learning. However, many researchers have shown easier ways to estimate energy with much more efficiency without estimating the partition function. 
\begin{figure}[htbp]
\centering
\includegraphics[width=0.65\columnwidth]{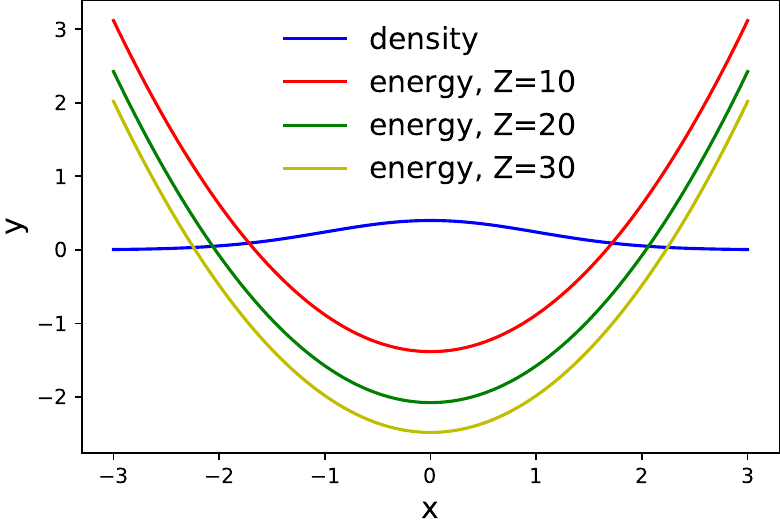}
\vspace{-9pt}
\caption{An example of the density and the energy. The blue line represent the density of $x\sim\mathcal{N}(0,1)$, and the other lines denote the energy of different partition function $Z$.}
\vspace{-4pt}
\label{fig:energy-example}
\end{figure}
\section{Energy-Based Imitation Learning}
\label{sec:ebil}
We begin by discussing a specific form of IL objectives. Since the different choices of $\psi$ in \eq{eq:rl-irl} lead to different distance metrics to solve IL, and by any kind of $f$-divergence $D_{f}(\rho_{\piE}\|\rho_{\pi})$ \cite{ghasemipour2019divergence}  corresponds to one kind of form of $\psi$ , we can let $\psi=\bbE_{\pi_E} [-1-\log(r(s, a))+r(s, a)]$ and we can have a reverse KL divergence objective\footnote{Full derivations can be found in Appendix C and D of \protect\cite{ghasemipour2019divergence} and without loss of generality we replace $c(s,a)$ with $-r(s,a)$.}:
\begin{equation}\label{eq:kl-il}
    \min_{\pi} \psi^*(\rho_\pi - \rho_{\pi_E}) = D_{f}(\rho_{\piE}\|\rho_{\pi}) = \kld(\rho_{\pi}\|\rho_{\piE})
\end{equation}

Before continuing, we first consider to model the normalized expert occupancy measure $\rho_{\piE}$ with Boltzmann distribution using an EBM $E_{\pi_E}(s,a)$:
\begin{equation}\label{eq:ebm}
    (1-\gamma)\rho_\piE(s,a) = \frac{1}{Z}\exp(-E_{\pi_E}(s,a))~.
\end{equation}

Then we take \eq{eq:ebm} into \eq{eq:kl-il} and manipulate trivial algebraic deviations\footnote{Full deviations can be found in \ap{ap:eb-il}.}:
\begin{equation}\label{eq:kl-induce}
    \begin{aligned}
        \kld ( \rho_\pi \| \rho_{\piE})
        &= \sum_{s,a}\rho_\pi\left( \log{\rho_\pi(s,a)} - \log{\rho_\piE(s,a)} \right )\\
        &\leq \bbE_\pi \left[E_\piE(s,a)\right ] - H(\pi) + \text{const}
        ~,
    \end{aligned}
\end{equation}
where $E_\piE$ is the EBM of policy $\piE$. Therefore, to minimize the KL divergence, we can choose to minimize its upper bound:
\begin{equation}\label{eq:eb-il}
\argmin_{\pi} \kld ( \rho_\pi \| \rho_{\piE}) \Rightarrow \argmax_{\pi} \bbE_\pi \left[-E_\piE(s,a)\right ] + H(\pi)~,
\end{equation}
which is exactly the objective of the MaxEnt RL (\eq{eq:maxent-rl}). It is worth noting that if we remove the entropy term to construct a standard RL objective, then it will collapse into minimizing the cross entropy of the occupancy measure rather than the reverse KL divergence.

Such an observation essentially shows us a two-stage imitation learning solution. In the first stage called \textit{energy recovery}, we try to estimate the energy of the expert's occupancy measure, and in the second stage we can utilize the recovered energy as a fixed surrogate reward for any \textit{MaxEnt RL} algorithm. It is worth noting that these learning procedures can be fully separated, i.e., the training of the energy and the policy do not have to be alternate! We call such a formulation as Energy-Based Imitation Learning (EBIL) framework which liberates us from designing complex learning algorithms but focus on the estimation of the expert energy $E_{\piE}(s,a)$, which is free to take recent advances in energy-based modeling. 

Except the way that utilizing the energy as the reward for reinforcement learning, one may notice that there is a simpler way to recover the expert policy directly via Gibbs distribution:
\begin{equation}
    \begin{aligned}
        \pi^*(a|s)&= \frac{\rho_{\piE}(s,a)}{\sum_{a'}\rho_{\piE}(s,a')}=\frac{\frac{1}{Z}\exp(-E_{\piE}(s,a))}{\frac{1}{Z}\sum_{a'}\exp{(-E_{\piE}(s,a'))}}\\
        &=\frac{\exp(-E_{\piE}(s,a))}{\sum_{a'}\exp{(-E_{\piE}(s,a'))}}~.
    \end{aligned}
\end{equation}
This indicates that if we accurately estimate the energy of the expert's occupancy measure, we have recovered the expert policy. However, this may be hard to generalize to continuous or high-dimensional action space but it worth to try in simple discrete domains.

\paragraph{State-only Imitation Learning}

It is also easy to derive an energy based algorithm for state-only imitation learning~\cite{sun2019provably,liu2019state,ghasemipour2019divergence}, where agents can only get the states (or observations) of the expert, by optimizing the following objective:
\begin{equation}
    \pi^* = \argmin_\pi \kld(\rho(s)\|\rho_{\piE}(s))~.
\end{equation}
Modeling the normalized state occupancy using Boltzmann distribution and following similar deviations, we can get an equivalent objective as \eq{eq:eb-il} by replacing the energy of state-action $E(s,a)$ with the energy of states $E(s)$.

In practice, we train the EBM from expert demonstrations as the reward function. Specifically, instead of directly using the estimated energy function $E(s,a)$, we can construct a surrogate reward function $\hat{r}(s,a) = h(-E_{\piE}(s,a))$, where $h(x)$ is a monotonically increasing linear function, which do not change the optimality in \eq{eq:eb-il} and can be specified differently for various environments. The step-by-step EBIL algorithm is presented in \alg{alg:EBIL}, which is simple and straightforward.

\section{Relation to Inverse Reinforcement Learning}
\label{sec:relation}
As shown in \eq{eq:irl}, MaxEnt IRL \cite{ziebart2008maximum} also models the trajectory distribution in an energy form, which reminds us to analyze the relation between EBIL and IRL.
Recent remarkable works focus on solving the IRL problem in an adversarial style \cite{ho2016generative,fu2017learning,finn2016guided} motivated by the progress in Generative Adversarial Nets (GANs) \cite{goodfellow2014generative}, which alternately optimize the reward function and a maximum-entropy policy corresponding to that reward.  In this section, we aims to understand why previous solutions need a alternating training between the policy and the reward, and what to do if we want a two-stage algorithm. We conclude that MaxEnt IRL can be seen as a special form of EBIL which alternately estimates the energy and EBIL can be a simplified alternative to adversarial IRL methods. 

Related theoretical connections among IRL, GANs have been thoroughly discussed in work by \citet{finn2016connection}, which can be bridged by EBMs. We borrow the insight from \citet{finn2016connection} and figure out the following questions: 
\begin{itemize}[leftmargin=19pt]
    \item[\textbf{Q1}:] Is the adversarial (or alternating) style necessary for IRL?
    \item[\textbf{Q2}:] If unnecessary, how can we avoid adversarial training?
    \item[\textbf{Q3}:] Is the adversarial style superior than other methods?
\end{itemize}

To answer Q1, we need to first identify the reason where the alternating training comes from. As revealed by \citet{finn2016connection}, the connections between MaxEnt IRL and GAN is derived from Guided Cost Learning (GCL) \cite{finn2016guided}, a sampling based method of IRL. Typically, to solve the cost function in \eq{eq:irl}, a maximum likelihood loss function is utilized:
\begin{equation}\label{eq:irl-cost}
\begin{aligned}
     \caL_{\text{cost}}(\theta) = \bbE_{\tau\sim p}[-\log p_{\theta}(\tau)]=\bbE_{\tau\sim p}[c_{\theta}(\tau)]+\log Z~.
\end{aligned}
\end{equation}
Notice that such a loss function needs to estimate the partition function $Z$, which is always hard for high-dimensional data. To that end, \citet{finn2016connection} proposed to train a new sampling distribution $q(\tau)$ and estimated the partition function $Z$ via importance sampling $\bbE_{\tau\sim q}\left [\frac{\exp{(-c_{\theta}(\tau))}}{q(\tau)}\right ]$.
In fact, the sampling distribution $p(\tau)$ corresponds to the agent policy, which is optimized to minimize the KL divergence between $q(\tau)$ and the $p_{\theta}(\tau)$:
\begin{equation}
     \caL_{\text{sampler}}(q) = \kld(q(\tau)\|p_{\theta}(\tau)) = \bbE_{\tau\sim q}[c_{\theta}(\tau)+\log q(\tau)]~.
\end{equation}
One may notice that these two optimization problems depend on each other, and thus lead to an alternating training scheme. However, a serious problem comes from the high variance of the importance sampling ratio and \citet{finn2016guided} applied a mixture distribution to alleviate this problem. Another different solution comes from \citet{ho2016generative}, which utilizes the typical unconstrained form of the discriminator without using the generator density, but the optimization of the discriminator is corresponding to the optimization of the cost function in \eq{eq:irl-cost}, as \citet{finn2016connection} proved. Therefore, we understand that the alternating IL algorithm suffers from the interdependence between the cost (reward) function and the policy due to the choice of maximum likelihood objective for the cost function. In this point of view, MaxEnt IRL can be seen as a special form of EBIL which estimates the energy alternately according to the following proposition. 

\begin{proposition}\label{prop:tau-kl}
    Suppose that we have recovered the optimal cost function $\hat{c}^*$, then minimizing the distance between trajectories is equivalent to minimizing the distance of occupancy measures: 
    \begin{equation}\label{eq:tau-kl}
        \begin{aligned}
        \argmin_\pi \kld(p(\tau)\|p(\tau_E)) = \argmin_\pi \kld(\rho(s,a)\|\rho_{\piE}(s,a)) ~.
        \end{aligned}
    \end{equation}
\end{proposition}
The proof is straightforward and we leave it in \ap{ap:tau-kl} for detailed checking. 

Therefore, the answer to Q1 is definitely not and with sufficient analysis we get our conclusion to Q2: we can avoid alternating training if we can decouple the dependence between the optimization of the reward $r(s,a)$ and the policy $\pi(a|s)$. Specifically, if we can \textit{learn the energy $E_{\piE}(s,a)$ without estimating the partition function $Z$}, or \textit{estimate the partition without using a learnable sampling policy}, we are free to change the alternate training style into a two-stage procedure. Therefore, EBIL and MaxEnt IRL are actually two sides of the same coin, and EBIL can be thought as an simplified alternative to adversarial IRL methods, which benefits from the flexibility and simplicity of EBMs. 

For Q3, we suggest the readers to refer to researches on probabilistic modeling~\cite{gutmann2010noise,lecun2006tutorial}. Although many recent works~\cite{goodfellow2014generative,bose2018adversarial} depend on an adversarially learned sampler, there are no theoretical guarantees showing that such an alternating learning style is more advanced than other streams of methods. Therefore, we can not theoretically get an answer to Q3 so leave the answer of Q3 in our quantitative and qualitative experiments where we compare our two-stage EBIL algorithms with those former adversarial IRL methods on various tasks.
\section{Expert Energy Estimation via Score Matching}
\label{sec:energy-estimation}

As described above, we desire a two-stage energy-based imitation learning algorithm, where the estimation of the energy function can be decoupled from learning the policy. In particular, the estimation of energy can either be done without estimating the partition function, or learned with an additional policy to approximate the partition, which is referred as adversarial IRL methods.
However, although various related work lie in the domain of energy based statistic modeling, they are not easy to be used for EBIL since many of them do not estimate the energy value $E(x)$ itself. For example, although the branch of score matching methods~\cite{vincent2011connection,song2020sliced} defines the score as the gradient of the energy, i.e., $\nabla_x\log p(x)=-\nabla_x E(x)$, estimate $-\nabla_x E(x)$ is enough for statistic modeling. To that end, in this paper, we refer to a recent denoising score matching work (e.g., DEEN~\cite{saremi2018deep}) that directly estimates the energy value through deep neural network in a differentiable framework.

Formally, let the random variable $X=(s,a)\sim \rho_\piE(s,a)$. 
Let the random variable $Y$ be the noisy observation of $X$ that $y\sim x+N(0, \sigma^2I)$, \textit{i.e.}, $y$ is derived from samples $x$ by adding with white Gaussian noise $\xi\sim N(0, \sigma^2I)$. 
The empirical Bayes least square estimator, \textit{i.e.}, the optimal denoising function $g(y)$ for the white Gaussian noise, is solved as
\begin{equation}
\begin{aligned}\label{eq:denoise}
    g(y) = y + \sigma^2\nabla_y\log{p(y)}~.
\end{aligned}
\end{equation}
Solving such a problem, we can get the score function $\nabla_y\log p(y)$. But remember we need the energy function instead, therefore a parameterized energy function $E_\theta(y)$ is modeled by a neural network explicitly. As shown in \cite{saremi2018deep}, such a DEEN framework can be trained by minimizing the following objective:
\begin{equation}\label{eq:deen}
    \argmin_{\theta} \sum_{x_i\in X, y_i\in Y} \left \| x_i - y_i + \sigma^2\frac{\partial E_\theta(y=y_i)}{\partial y} \right \|^2~,
\end{equation}
which ensures the relation of score function $\partial E_\theta(y)/\partial y$ shown in \eq{eq:denoise}. It is worth noting that the EBM estimates the energy of the noisy samples. This can be seen as a Parzen window estimation of $p(x)$ with variance $\sigma^2$ as the smoothing parameter~\cite{saremi2019neural,vincent2011connection}. A trivial problem here is that \eq{eq:deen} requires the samples (state-action pairs) to be continuous so that the gradient can be accurately computed. Practically, such a problem can be solved via state/action embedding or using other energy estimation methods, \textit{e.g.}, Noise Contrastive Estimation~\cite{gutmann2010noise}. 

In practice, we learn the EBM of expert data from offline demonstrations and construct the reward function, which will be fixed until the end to help agent learn its policy with a normal RL procedure. Specifically, we construct the surrogate reward function $\hat{r}(s,a)$ as follows:
\begin{equation}\label{eq:reward}
    \hat{r}(s,a) = h(-E_{\piE}(s,a))~,
\end{equation}
where $h(x)$ is a monotonically increasing linear function, which can be specified for different environments.
Formally, the overall EBIL algorithm is presented in \alg{alg:EBIL}.

\label{ap:algo}
\begin{algorithm}[htbp]
    \caption{Energy-Based Imitation Learning}
    \label{alg:EBIL}
    \begin{algorithmic}[1]
       \STATE {\bfseries Input:} Expert demonstration data $\tau_E=\{ (s_i, a_i) \}_{i=1}^{N}$,  parameterized energy-based model $E_{\phi}$, parameterized policy $\pi_{\theta}$;
       
       \FOR{$k = 0, 1, 2, \dotsc$}
       \STATE Optimize $\phi$ with the objective in \eq{eq:deen}.
       \ENDFOR
       
       Compute the surrogate reward function $\hat{r}$ via \eq{eq:reward}.
       
       \FOR{$k = 0, 1, 2, \dotsc$}
       \STATE Update $\theta$ with a normal RL procedure using the surrogate reward function $\hat{r}$.
       \ENDFOR
       
     \RETURN {$\pi$}
    \end{algorithmic}
\end{algorithm}
\section{Related Work}
\subsection{Imitation Learning}
As extensions for the traditional solution as inverse reinforcement learning~\cite{ng2000algorithms,abbeel2004apprenticeship,fu2017learning}, generative adversarial algorithms have been raised up since years ago. Tracing back to GAIL, which models the imitation learning as an occupancy measure matching problem, and takes a GAN-form objective to optimize both the reward and the policy~\cite{ho2016generative}. After that, Adversarial Inverse Reinforcement Learning (AIRL) simplifies the idea of \citet{finn2016connection} and use a disentangled discriminator to recover the reward function in an energy-form~\cite{fu2017learning}. However, as we show in experiments, they do not actually recover the energy. Based on previous works, \citet{ke2019imitation} and \citet{ghasemipour2019divergence} concurrently proposed to unify the adversarial learning algorithms with f-divergence. Like \citet{nowozin2016f}, they claimed that any $f$-divergence can be used to construct a generative adversarial imitation learning algorithm. Specifically, \citet{ghasemipour2019divergence} proposed FAIRL, which adopts the forward KL as the distance metric.

Instead of seeking to alternatively update the policy and the reward function as in IRL and GAIL, many recent works of IL aim to learn a fixed reward function directly from expert demonstrations and then apply a normal reinforcement learning procedure with that reward function. This idea can be found inherently in Generative Moment Matching Imitation Learning (GMMIL)~\cite{kim2018imitation} that utilized the maximum mean discrepancy as the distance metric to guide training. Recently, \citet{wang2019red} proposed Random Expert Distillation (RED), which employs the idea of Random Network Distillation\cite{burda2018rnd} to compute the reward function by the loss of fitting a random neural network or an auto-encoder. \citet{reddy2019softqil} applied constant rewards by setting positive rewards for the expert state-actions and zero rewards for other ones, which is optimized with the off-policy Soft Q-Learning (SQL) algorithm~\cite{haarnoja2018soft}. In addition, Disagreement-Regularized Imitation Learning (DRIL) \cite{brantley2020disagreementregularized} constructs the reward function using the disagreement in their predictions of an ensemble of policies trained on the demonstration data, which is optimized together with a supervised behavioral cloning cost. 
These works resembles the idea of EBIL, where the reward function can be seen as approximated energy functions. For example, RED \cite{wang2019red} estimates the reward using an auto-encoder, which is exactly the way of energy modeling in EBGAN~\cite{zhao2016energy}. In addition, \citet{wang2019red} and \citet{brantley2020disagreementregularized} also utilized the prediction errors, which are low on demonstration data but high on data that is out of the demonstration (similar to energy). The method of \citet{reddy2019softqil} is more straightforward, which simply sets the energy of the expert as 0 and the agent's as 1. However, these approximated energies are not derived from statistical modeling and lack theoretical 
correctness. 
It is worth noting that combining the idea of EBGAN~\cite{zhao2016energy} and GAIL~\cite{ho2016generative} we can also design an adversarial style energy-based algorithm for imitation learning, however it is less interesting and also does not have statistical support about the learned energy.

\subsection{Energy-Based Modeling}
Our EBIL relies highly on EBMs, which have played an important role in a wide range of tasks including image modeling, trajectory modeling and continual online learning \cite{Du2019generalization}. 
Thanks to the appealing features, EBMs have been introduced into many RL literature, for instance, parameterized as value function~\cite{sallans2004valueEBM}, employed in the actor-critic framework~\cite{heess2012acEBMs}, applied to MaxEnt RL~\cite{haarnoja2017reinforcement} and model-based RL regularization~\cite{boney2019regularizing}. 
However, EBMs are always difficult to train due to the partition function \cite{finn2016connection}. Nevertheless, recent works have tackled the problem of training large-scale EBMs on high-dimensional data, such as DEEN~\cite{saremi2018deep} which is applied in our implementation. Except for DEEN, there still leave plenty of choices for efficiently training EBMs \cite{gutmann2012noise,Du2019generalization,nijkamp2019learn}. 

\section{Experiments}
\label{sec:exps}

In this section we seek to empirically evaluate EBIL to figure out the effectiveness of our solution compared to previous works, especially the alternative optimization methods. We first conduct qualitative and quantitative experiments on a simple one-dimensional environment, where we illustrate the recovered reward in the whole state-action space and show the training stability of EBIL. Then, we test EBIL against baselines on benchmark environments using sub-optimal experts and release competitive performance.

\subsection{Synthetic Task}
\label{sec:one-domain}

\begin{figure}
\includegraphics[width=0.8\linewidth]{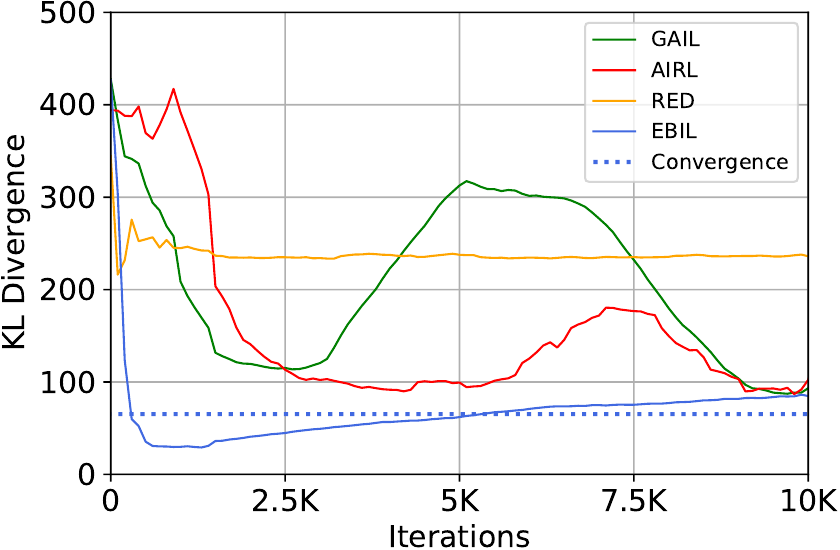}
\vspace{-2pt}
\caption{The KL divergence between the agent trajectories and the expert during the learning procedure, which indicates that EBIL is much more stable than the other methods. The blue dash line represents the converged result of EBIL.}
\vspace{-14pt}
\label{fig:kl-div}
\end{figure}

In the synthetic task, we want to evaluate the qualitative performance of different IL methods by displaying the heat map of the learned reward signals and sampled trajectories. 
We want EBIL to be capable of guiding the agent to recover the expert policy and correspondingly generate the high-quality trajectories. Therefore, we evaluate EBIL on a synthetic environment where the agent tries to move in a one-dimensional space. 
Specifically, the state space is $[ -0.5, 10.5 ]$ and the action space is $[-1, 1]$. The environment initializes the state at $0$, and we set the expert policy as static rule policies $\pi_E = \gN(0.25, 0.06)$ when the state $s\in [ -0.5, 5 )$, and $\pi_E = \gN(0.75, 0.06)$ when $s\in [ 5, 10.5 ]$. The sampled expert demonstration contains $40$ trajectories with up to $30$ timesteps in each one. 
For all methods, we choose Soft Actor Critic (SAC) \cite{haarnoja2018soft} as the learning algorithm and we continue training each algorithm until convergence.

In this illustrative experiment, we compare EBIL aginst GAIL~\cite{ho2016generative}, AIRL~\cite{fu2017learning} and RED~\cite{wang2019red}, where GAIL and AIRL are two representative works of adversarial imitation learning which take an alternative updating on the reward and the policy. RED resembles EBIL which also take a two-stage training by first estimate the reward through the prediction error of a trained network with a randomized one. We plot the KL divergence between the agent's and the expert's trajectories during the training procedure in \fig{fig:kl-div} and visualize the final estimated rewards with corresponding induced trajectories in \fig{fig:heat_maps}.

As illustrated in \fig{fig:ebil_heat} and \fig{fig:kl-div}, the pre-estimated reward of EBIL successfully captures the density of the expert trajectories, and led by the interpretative reward the induced policy is able to quickly converge to the expert policy. By contrast, GAIL requires a noisy adversarial process to correct the policy. As a result, although GAIL achieves compatible final performance against EBIL~(\fig{fig:gail_heat}), it suffers a slow, unstable training as shown in \fig{fig:kl-div} and assigns meaningless reward in the state-action space, as revealed in \fig{fig:gail_heat}.
In addition, we are surprised to find that AIRL recovers an `inverse'-type reward signals but still learns a good policy, as suggested in \fig{fig:airl_heat}. We analyze such a problem in \ap{ap:airl}, where we conclude that AIRL actually does not recover the energy of experts but the energy with an entropy term. Removing such an entropy term makes the result more reasonable. Finally, we notice that constructing the reward as the prediction errors does not help to recover a meaningful signal, as shown in \fig{fig:red_heat}, where RED suffers from the diverged reward and fails to imitate the expert accurately.

On the contrary, EBIL benefits from meaningful rewards which is tend to learn a deterministic policy which shows the most ood. Thus, the recovered policy has less variance but the mean is very close to the expert, which is a good property. In fact, an optimal policy is usually a deterministic policy due to Bellman equation and a stochastic policy is always used to increase the exploration ability in a learning procedure. the exploration ability is much necesarry for methods as EBIL and RED. Although it is convincing to take the recovered energy as a meaningful reward signal for imitation learning, we must point out that as illustrated in \fig{fig:ebil_heat}, the recovered reward maybe be sparse in the space where the expert scarcely comes by. Thus, to gain better performance, the agents will be required to equip with an efficient RL algorithm with good exploration ability.

\begin{figure*}[!t]
\centering
\subfigure[Expert Trajectories]{
\begin{minipage}[b]{0.28\linewidth}
\label{fig:expert_heat}

\includegraphics[width=1\linewidth]{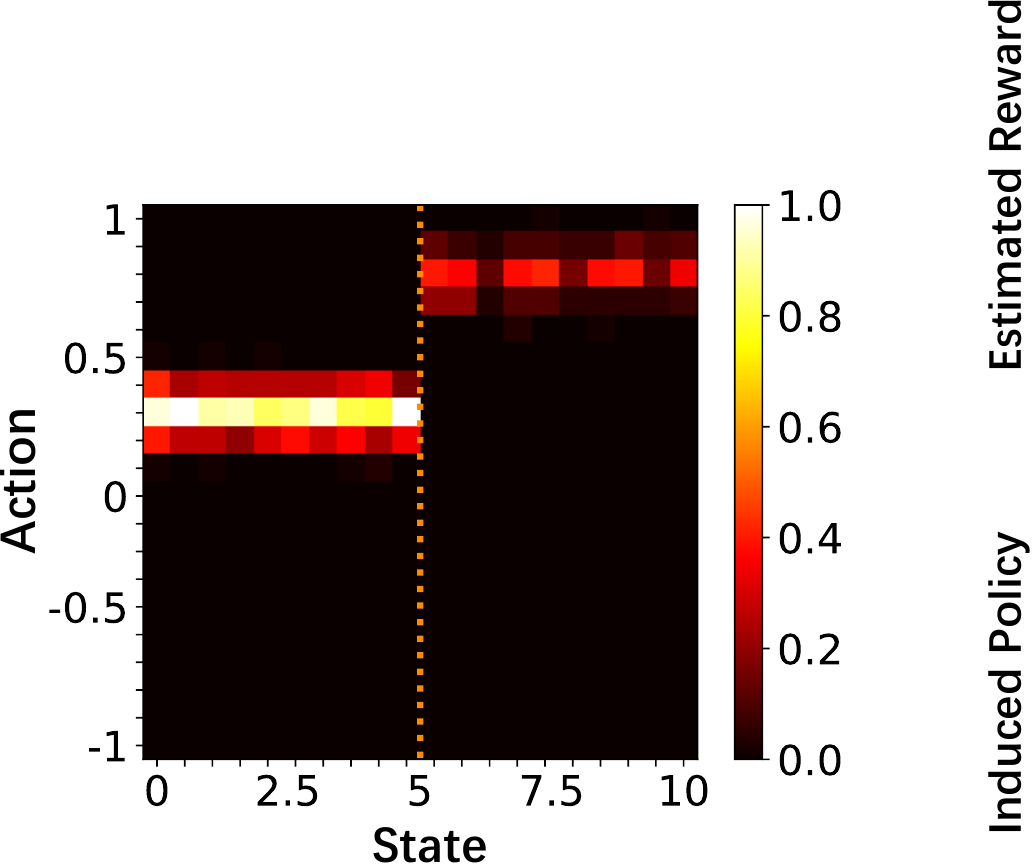}
\vspace{-0.8pt}
\end{minipage}
}
\hspace{-10pt}
\subfigure[EBIL]{
\begin{minipage}[b]{0.155\linewidth}
\label{fig:ebil_heat}
\includegraphics[width=1\linewidth]{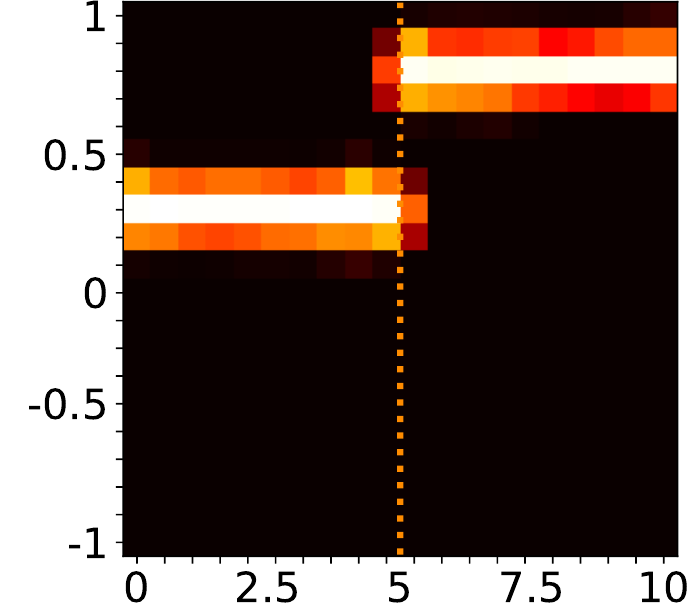}
\includegraphics[width=1\linewidth]{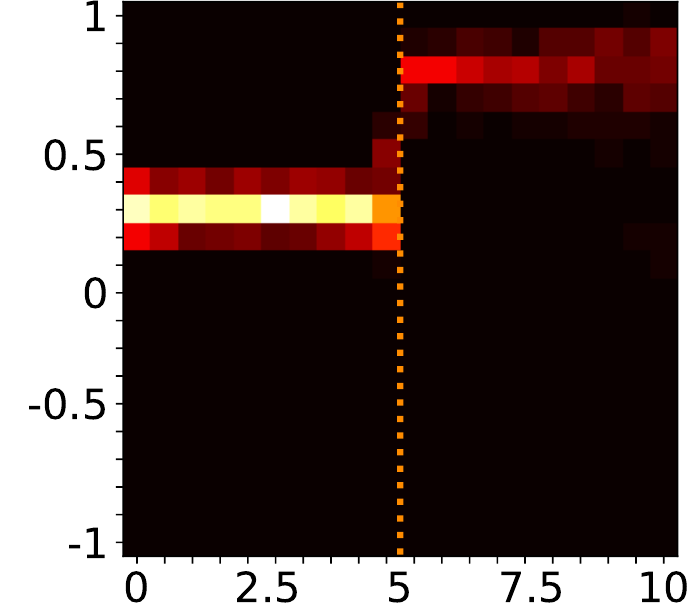}
\end{minipage}
}
\subfigure[GAIL]{
\begin{minipage}[b]{0.155\linewidth}
\label{fig:gail_heat}
\includegraphics[width=1\linewidth]{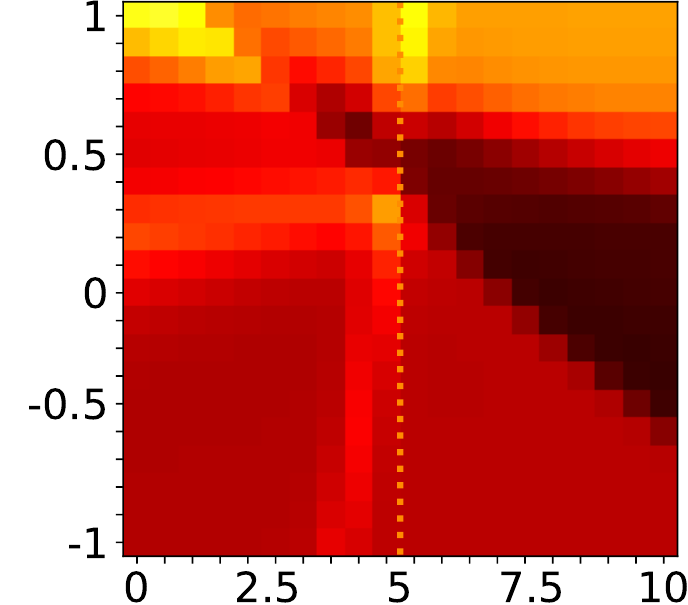}
\includegraphics[width=1\linewidth]{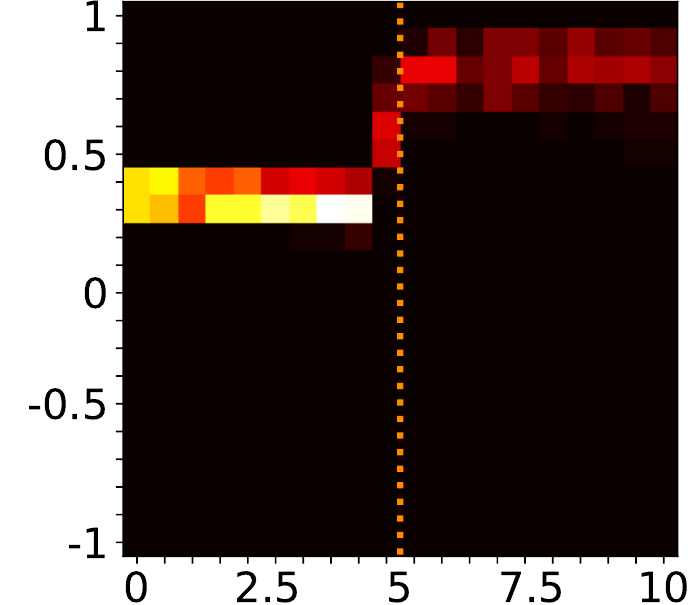}
\end{minipage}
}
\subfigure[AIRL]{
\begin{minipage}[b]{0.155\linewidth}
\label{fig:airl_heat}
\includegraphics[width=1\linewidth]{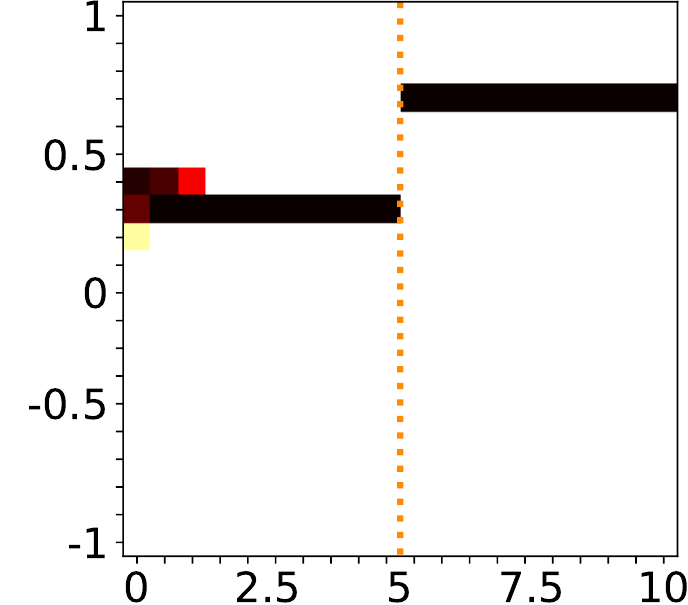}
\includegraphics[width=1\linewidth]{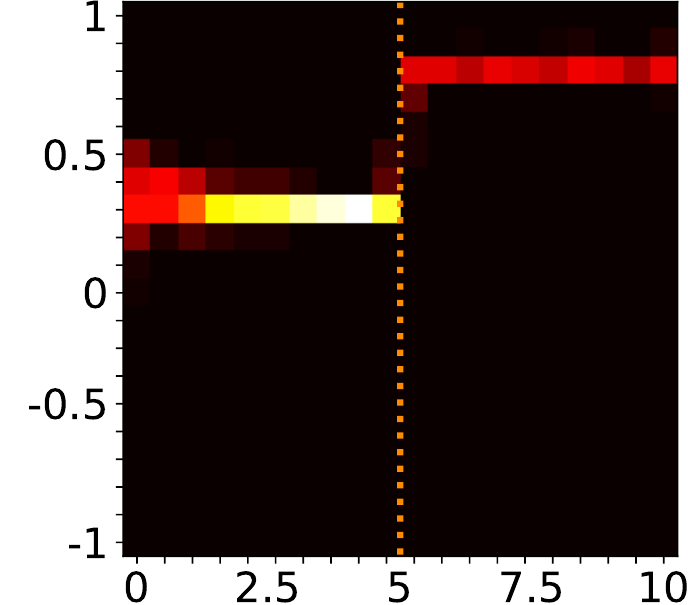}
\end{minipage}
}
\subfigure[RED]{
\begin{minipage}[b]{0.155\linewidth}
\label{fig:red_heat}
\includegraphics[width=1\linewidth]{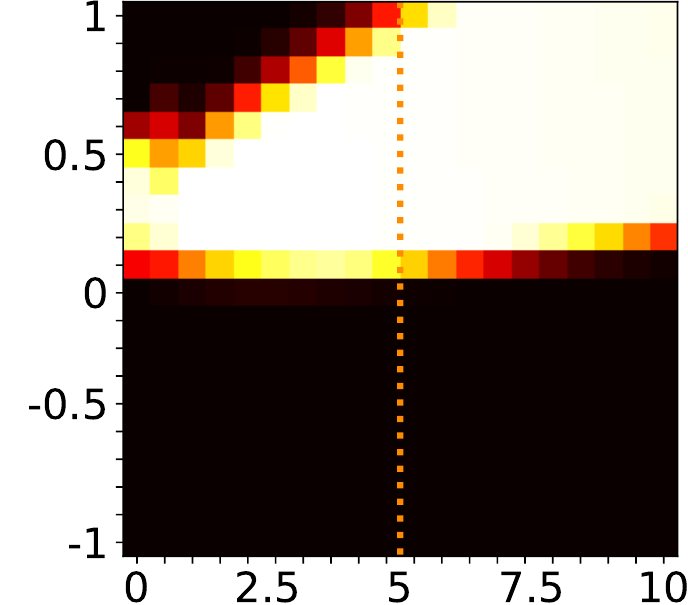}
\includegraphics[width=1\linewidth]{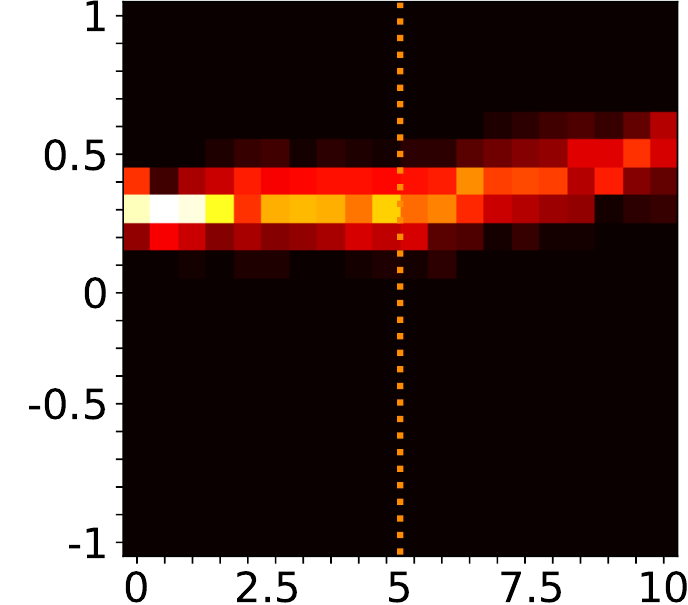}
\end{minipage}
}
\vspace{-5pt}
\caption{Heat maps of the expert trajectories (leftmost), the (final) \textit{estimated rewards} recovered by different methods (top) and the corresponding \textit{induced policy} (bottom). The horizontal and the vertical axis denote the \textit{state space} and the \textit{action space} respectively. The red dotted line represents the position where the agent should change its policy. 
It is worth noting that EBIL and RED both learn fixed reward functions, while GAIL and AIRL iteratively update the reward signals. We do not compare BC since it learns the policy via supervised learning without recovering any reward signals.}
\vspace{-2pt}
\label{fig:heat_maps}
\end{figure*}


\subsection{Imitation on Continous Control Benchmarks}

\paragraph{Sub-optimal Demonstrations}

We conduct evaluation experiments on the standard continuous control benchmarking MuJoCo environments. For each task, we train experts with OpenAI baselines version~\cite{dhariwal2016openai} of Proximal Policy Optimization (PPO)~\cite{schulman2017proximal} and get sub-optimal demonstrations. We employ Trust Region Policy Optimization (TRPO)~\cite{schulman2015trust} as the learning algorithm in the implementation for all evaluated methods, and we do not apply BC initialization for all tasks. We consider 4 demonstrated trajectories by the sub-optimal expert, as \cite{ho2016generative,wang2019red} do. We compare EBIL with several baseline methods and report the \textit{converged} result in \tb{tab:mujoco}, which are evaluated throughout 50 test episodes. 

As shown in \tb{tab:mujoco}, with sub-optimal demonstrations, EBIL can achieve the best or comparable performance among all environments, indicating that the recovered energy is able to become a good reward signal for imitating the experts. However, learning from such a sub-optimal expert seems challenging for GAIL and AIRL, which are less robust to reach a better performance. We notice that on some environment as InvertedDoublePendulum, GAIL and AIRL can work well during the training, but does not converge until the end. We think the problem is due to the training instability of GAN, which always need an early stop to get a better performance. On the contrary, EBIL provide steady reward signals  instead of a alternative trained one that stabilizes the training. The performance of RED and GMMIL also indicates that the recovered reward by random distillation or the maximum mean discrepancy do not provide a stable guidance in such an sub-optimal setting. However, benefit from such a steady reward function, RED and GMMIL can perform better than GAIL and AIRL on some tasks.

\begin{table*}[t]
\vskip 0.15in
\caption{Comparison for different methods of the episodic true rewards with sub-optimal demonstrations on 5 continuous control benchmarks. The means and the standard deviations are evaluated over different random seeds.}
\vspace{-3pt}
\begin{center}
\begin{tabular}{cccccc}
\toprule
Method & Humanoid & Hopper & Walker2d & Swimmer & InvertedDoublePendulum\\
\midrule
Random & 100.38 $\pm$ 28.25 & 14.21 $\pm$ 11.20 & 0.18 $\pm$ 4.35 & 0.89 $\pm$ 10.96 & 49.57 $\pm$ 16.88 \\
\hline
BC & 178.74 $\pm$ 55.88 & 28.04 $\pm$ 2.73 & 312.04 $\pm$ 83.83 & 5.93 $\pm$ 16.77 & 138.81 $\pm$ 39.99\\
GAIL & 299.52 $\pm$ 81.52  & \textbf{1673.32 $\pm$ 57.13}  & 329.01 $\pm$ 211.84  & 23.79 $\pm$ 21.84  & 327.42 $\pm$ 94.92 \\
AIRL & 286.63 $\pm$ 6.05 & 126.92 $\pm$ 62.39  & 215.79 $\pm$ 23.04  & -13.44 $\pm$ 2.69 & 76.78 $\pm$ 19.63 \\
GMMIL & 416.83 $\pm$ 59.46 & 1000.87 $\pm$ 0.87 & 1585.91 $\pm$ 575.72 & -0.73 $\pm$ 3.28 & 4244.63 $\pm$ 3228.14\\
RED & 140.23 $\pm$ 19.10 & 641.08 $\pm$ 2.24 & 641.13 $\pm$ 2.75 & -3.55 $\pm$ 5.05 & 6400.19 $\pm$ 4302.03 \\
\textbf{EBIL} & \textbf{472.22 $\pm$ 107.72} & 1040.99 $\pm$ 0.53 & \textbf{2334.55 $\pm$ 633.91} & \textbf{58.09 $\pm$ 2.03} & \textbf{8988.37$\pm$ 1812.76}\\
\hline
Expert (PPO) & 1515.36 $\pm$ 683.59 & 1407.36 $\pm$ 176.91 & 2637.27 $\pm$ 1757.72 &  122.09  $\pm$ 2.60 &  6129.10 $\pm$ 3491.47 \\
\bottomrule
\end{tabular}
\label{tab:mujoco}
\end{center}
\end{table*}

\paragraph{Optimal Demonstrations}

We also want to know how EBIL performs on optimal demonstrations compared with previous methods, especially adversarial inverse reinforcement learning methods such as GAIL, AIRL, and FAIRL. 
Therefore we evaluate EBIL on optimal demonstrations from \cite{ghasemipour2019divergence}, where the expert is trained by SAC. As before, the demonstration for each task contains 4 trajectories and each trajectory is subsampled by a factor of 20. Similar to \cite{ghasemipour2019divergence}, we finetune each model and checkpoint the model at its \textit{best} validation loss and report the best resulting checkpoints on 50 test episodes. As shown in \tb{tab:mujoco-optimal}, we find that those adversarial algorithms can always achieve better performances on high-dimensional tasks as Hopper and Walker2d, where EBIL remains a gap between these methods. We think this problem is due to the accuracy of the energy model trained by DEEN, which only takes a simple MLP without further regularization operations. We must admit that benefiting from the various improvements of GAN such as gradient penelty~\cite{gulrajani2017improved} that makes the training more stable, adversarial algorithms has advantages over the traditional statistical modeling of EBMs, especially in high-dimensional space, since they can continue to improve the learning of data distribution in the iterative training procedure, while score matching methods as DEEN only model the provided dataset from scratch. 
However, on an easier environment LunarLander where the EBM can provide many meaningful and dense rewards, EBIL outperforms the others while all the three adversarial algorithms have large variances. 

\begin{table}[htbp]
\vskip 0.15in
\caption{Comparison for different methods of the episodic true rewards with optimal demonstrations on 3 continuous control benchmarks. The means and the standard deviations are evaluated over different random seeds. We notice that EBIL can work better than adversarial inverse reinforcement learning algorithms on low-dimensional environments as LunarLander but remains performance gaps on harder tasks as Hopper and Walker2d.}
\vspace{-2pt}
\begin{center}
\begin{tabular}{cccccc}
\toprule
Method & LunarLander & Hopper & Walker2d \\
\hline
Random & -232.81 $\pm$ 139.72 & 14.21 $\pm$ 11.20 & 0.18 $\pm$ 4.35 \\
\midrule
GAIL & -85.85 $\pm$ 59.22  & 3117.50 $\pm$ 2.96  & 4092.86 $\pm$ 7.46 \\
AIRL & -66.07 $\pm$ 104.14 & \textbf{3398.72 $\pm$ 8.39}  & 3987.23  $\pm$ 334.04 \\
FAIRL & -116.90 $\pm$ 15.79 & 3353.78 $\pm$ 9.12  & \textbf{4225.66  $\pm$ 65.11} \\
\textbf{EBIL} & \textbf{237.95 $\pm$ 44.96} & 2401.93 $\pm$ 6.85 & 3026.60 $\pm$ 57.14 \\ 
\hline
Expert & 254.90 $\pm$ 24.11 & 3285.92 $\pm$ 2.14 & 4807.22 $\pm$ 166.35\\
\bottomrule
\end{tabular}
\label{tab:mujoco-optimal}
\end{center}
\vspace{-3pt}
\end{table}

\begin{figure*}[htbp]
\centering
\subfigure[Target State Marginal Distribution]{
\begin{minipage}[b]{0.23\linewidth}
\label{fig:expert-smm}
\includegraphics[width=1\linewidth]{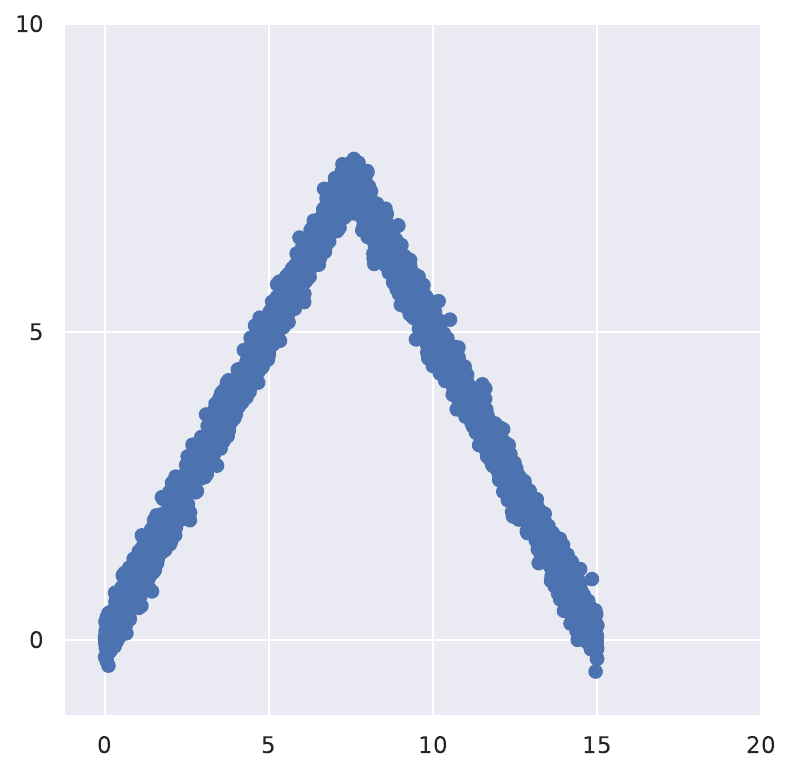}
\end{minipage}
}
\subfigure[Recovered Energy]{
\begin{minipage}[b]{0.23\linewidth}
\label{fig:energy-smm}
\includegraphics[width=1\linewidth]{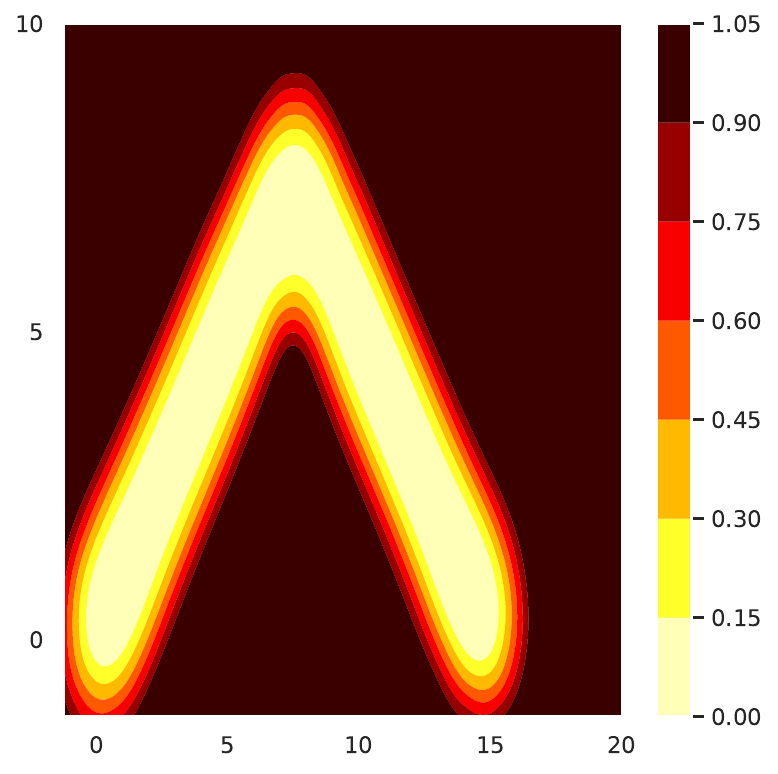}
\end{minipage}
}
\subfigure[Induced Policy]{
\begin{minipage}[b]{0.23\linewidth}
\label{fig:policy-smm}
\includegraphics[width=1\linewidth]{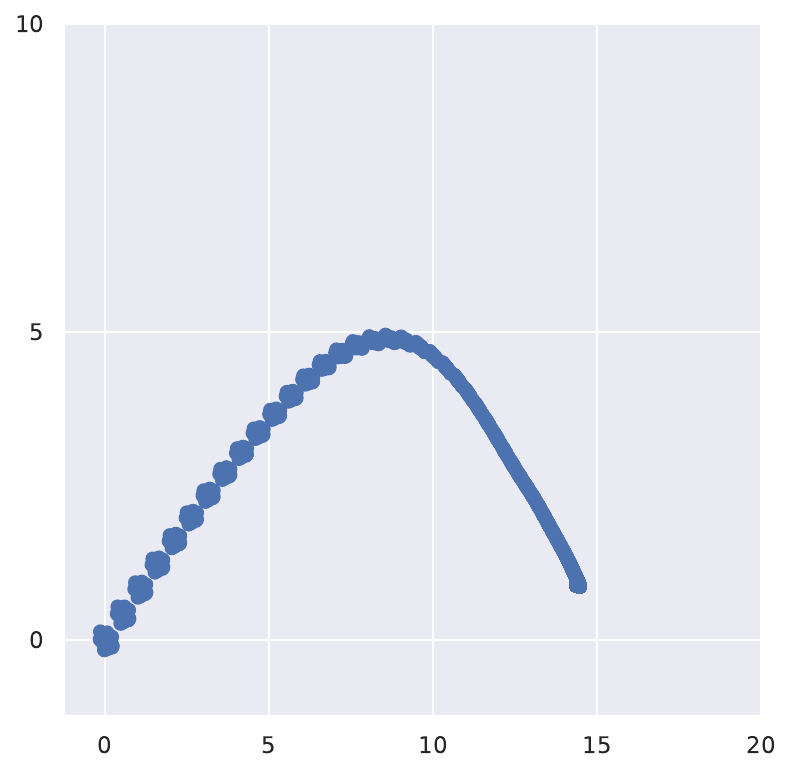}
\end{minipage}
}
\vspace{-8pt}
\caption{(a) Heuristically-designed samples from target state marginal distribution. (b) The recovered energy by EBIL, representing the density of the demonstrations, where the darker color represent the lower energy and higher rewards. (c) The induced policy using the state-only reward constructed from the energy as shown in (b). The energy value is in a range of $[0,1]$ since we use \textit{sigmoid} for the last layer of the DEEN network on this task.}
\vspace{-2pt}
\label{fig:smm}
\end{figure*}

\begin{figure}[htbp]
\centering
\includegraphics[width=0.49\columnwidth]{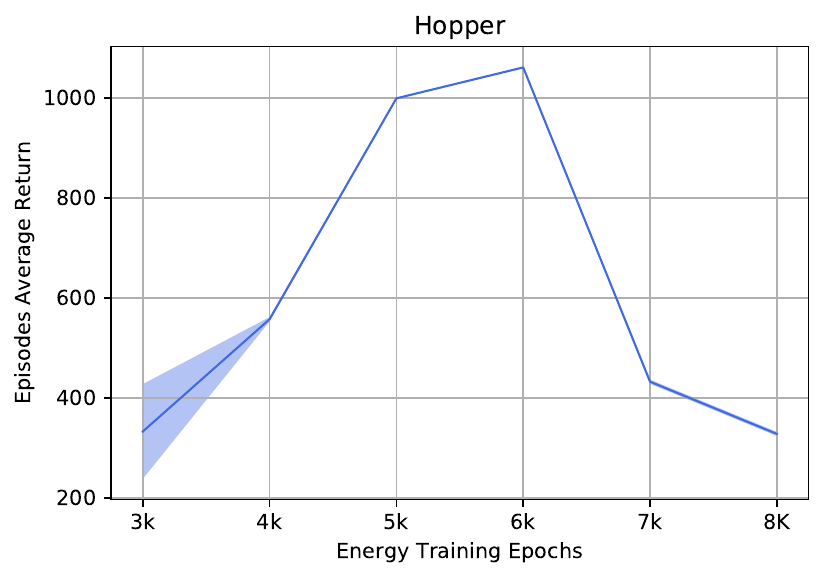}
\includegraphics[width=0.49\columnwidth]{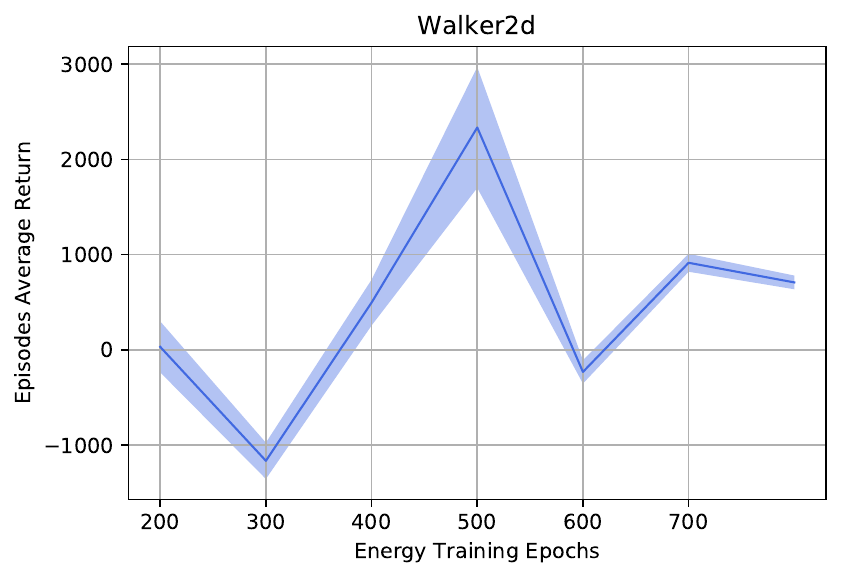}
\vspace{-7pt}
\caption{Ablation study on the average episode rewards evaluated on Hopper and Walker2 that are learned with energy models from different training epochs using sub-optimal demonstrations, where the solid line and the shade represent the mean and the standard deviation on different runs separately.}
\vspace{-2pt}
\label{fig:ablation-energy}
\end{figure}

\subsection{State Marginal Matching}

In this section, we show that EBIL can also be effective in matching the state marginal distributions for a given dataset. Motivated by \citet{ghasemipour2019divergence}, we try to make the agent learn the desired policy without expert demonstrations. Specifically, 
unlike the traditional IL setting, the target state marginal distribution does not even have to be a realizable state-marginal distribution collected by expensive expert demonstrations, but easy-to-get heuristically-designed interpretable distributions. Therefore, we test our methods in a 2D point mass task similar to ~\cite{ghasemipour2019divergence}, where the agent can move around the ground. We first generate 4,000 heuristic samples as shown in \fig{fig:expert-smm}, learn the energy from them and then train the policy guided by the energy.

As expected, \fig{fig:energy-smm} illustrates that DEEN successfully recovers the meaningful reward function, which describes the density of the target state marginal distribution. We then construct the reward function as $r(s) = h(-E(s))$ where $h$ is a monotonic linear function, with which the agent is finally able to induce similar trajectories to the expert. It is worth noting that we shape the recovered reward with an encourage on the exploration on the x-axis, since we find that with the state-only reward, the agent tends to get stuck, i.e., learn to go forth and back at the very beginning near the starting point. Such a task reminds us of another possibility of finishing reinforcement learning tasks, where people do not have to carefully design the reward function but the samples on the desired trajectories. And by learning from the recovered energy of the provided samples with simple intuitive shaping that prevents the agent from cycling around, the agent can finally learn to complete the task.

\subsection{Ablation Study}
\label{sec:abl-study}

It is worth noting that in our experiments, we also find that a well-trained energy network may be hard for agents to learn the expert policies on some environments. 
We regard it as the ``sparse reward signals'' problem, as discussed in \se{sec:one-domain}. By contrast, sometimes a ``half-trained'' model may provide smoother rewards, and can help the agent to learn more efficiently. Actually, a similar phenomenon also occurs when training the discriminator in GAN. 
Therefore, to further understand the effect of the energy model from different training epochs, we conduct an ablation study on energy models trained from different epochs. The results are illustrated in \fig{fig:ablation-energy}, which verifies our intuition that a ``half-trained'' model can provide smoother rewards that solve the ``sparse reward'' problem, which is better for imitating the expert policy.

\section{Conclusion}

In this paper, we propose Energy-Based Imitation Learning (EBIL), which shows that it is feasible to compute a fixed reward function via directly estimating the expert's energy to help agents learn from the demonstrations, which breaks out of the traditional iterative IRL paradigm. We further discuss the relation of our method with Maximum Entropy Inverse Reinforcement Learning (MaxEnt IRL) and reveal that these methods are actually two sides of the same coin, where EBIL can be regarded as a simplified alternative of adversarial IRL methods. We conduct quantitative and qualitative evaluations in multiple tasks, showing that with recovering meaningful rewards, EBIL can finally lead to comparable performance against previous algorithms. For future work, we can try different energy estimation methods for expert demonstrations, and exploring more possibilities to utilize EBM to help the agent to learn in different reinforcement learning tasks as state-only learning. It is also intriguing to expand EBIL into multi-agent environments and construct interpretative reward signals for multi-agent learning.

\section*{Acknowledgement}

Weinan Zhang is supported by ``New Generation of AI 2030'' Major Project (2018AAA0100900) and National Natural Science Foundation of China (62076161, 61772333, 61632017). 
The author Minghuan Liu is supported by Wu Wen Jun Honorary Doctoral Scholarship, AI Institute, Shanghai Jiao Tong University.
The work is also supposed by MSRA Joint Research Grant.

\bibliographystyle{ACM-Reference-Format}
\balance
\bibliography{ref}

\clearpage
\onecolumn
\newpage
\appendix

\section{Proofs}
\subsection{Trivial Algebraic Deviations}
\label{ap:eb-il}
In \se{sec:ebil} we show that with an EBM we can have $\kld ( \rho_\pi \| \rho_{\piE}) = \bbE_\pi \left[E_\piE(s,a)\right ] - H(\pi) + \text{const}$, which can be manipulated with trivial deviations. Before showing the equivalence, we first present the following lemma which shows the definition of the entropy of the occupancy measure $\rho(s,a)$.
\begin{lemma}[Lemma 3 of \cite{ho2016generative}]
   $\overline{H}$ is strictly concave, and for all $\pi \in \Pi$ and $\rho \in \caD$, we have $H(\pi) = \overline{H}(\rho_\pi)$ and $\overline{H}(\rho)=H(\pi_\rho)$~,
where $\overline{H}\left (\rho \right )=-\sum_{s,a}\rho_\pi\log{\rho_\pi(s,a)/\sum_{a'}\rho_\pi(s,a')}$ is the entropy of the occupancy measure.
\end{lemma}

Take \eq{eq:ebm} into \eq{eq:kl-il} for policy $\piE$, one can obtain that: 
\begin{equation}\label{eq:kl-induce}
    \begin{aligned}
        \kld ( \rho_\pi \| \rho_{\piE}) &= \sum_{s,a} \rho_\pi(s,a) \log \frac{\rho_\pi(s,a)}{\rho_\piE(s,a)}\\
        &= \sum_{s,a}\rho_\pi\left( \log{\rho_\pi(s,a)} - \log{\frac{e^{-E_\piE(s,a)}}{(1-\gamma)Z'}} \right )\\
        &= \sum_{s,a}\rho_\pi\left(E_\piE(s,a)+\log{\rho_\pi(s,a)} + \log{(1-\gamma)Z'}\right )\\
        &= \bbE_\pi \left[E_\piE(s,a)\right ]+\sum_{s,a}\rho_\pi\log{\rho_\pi(s,a)} + \text{const}\\
        &= \bbE_\pi \left[E_\piE(s,a)\right ]+\sum_{s,a}\rho_\pi\left (\log{\rho_\pi(s,a)} 
        -\log{\sum_{a'}\rho_\pi(s,a')} + \log{\sum_{a'}\rho_\pi(s,a')}\right ) + \text{const}\\
        &= \bbE_\pi \left[E_\piE(s,a)\right ] +\sum_{s,a}\rho_\pi\log{\left [ \rho_\pi(s,a)/\sum_{a'}\rho_\pi(s,a')\right ]} + \sum_{s}\rho_\pi(s)\log{ \rho_\pi(s)} + \text{const}\\
        &= \bbE_\pi \left[E_\piE(s,a)\right ] - \overline{H}\left (\rho_\pi \right ) - H(\rho_{\pi}(s)) + \text{const}
        \\
        &\leq \bbE_\pi \left[E_\piE(s,a)\right ] - H(\pi) + \text{const}
        ~,
    \end{aligned}
\end{equation}
where $E_\piE$ is the EBM of policy $\piE$ and $Z'$ is its partition function. Therefore \eq{eq:kl-il} in the end leads to the objective function of EBIL \eq{eq:eb-il}:
\begin{equation}\label{eq:kl-eq-ebil}
    \begin{aligned}
    \argmin_{\pi} \kld ( \rho_\pi \| \rho_{\piE}) \Rightarrow \argmax_{\pi} \bbE_\pi \left[-E_\piE(s,a)\right ] + H(\pi)
    \end{aligned}
\end{equation}

\subsection{Proof of \protect\prop{prop:tau-kl}}
\label{ap:tau-kl}

\begin{proof}[Proof of \protect\prop{prop:tau-kl}]
Suppose we have recovered the optimal reward function $\hr$, then we can derive the objective of the KL divergence between the two trajectories into the forward MaxEnt RL procedure.

With chain rule, the induced trajectory distribution $p(\tau)$ is given by
\begin{equation}
     p(\tau) = p(s_0)\prod_{t=0}^{T}P(s_{t+1}|s_t,a_t)\pi(a_t|s_t)~.
\end{equation}

Suppose the desired expert trajectory distribution $p(\tau_E)$ is given by
\begin{equation}
\begin{aligned}
p(\tau_E) &\propto p(s_0)\prod_{t=0}^{T}P(s_{t+1}|s_t,a_t) \exp(\hr^*(\tau))\\
&= p(s_0)\prod_{t=0}^{T}P(s_{t+1}|s_t,a_t) \exp(\sum_{t=0}^{T}\hr^*(s_t, a_t))~,
\end{aligned}
\end{equation}
now we will show that the following optimization problem is equivalent to a forward MaxEnt RL procedure given the optimal reward $\hr^*$:
\begin{equation}\label{eq:tau-kl-induce-1}
\begin{aligned}
    \kld(p(\tau)\|p(\tau_E))&= \sum_{\tau\sim\pi} p(\tau) \log \frac{p(\tau)}{p(\tau_E)}\\
        &= \sum_{\tau\sim\pi}p(\tau)\left( \log{p(\tau)} - \log{p(\tau_E)} \right )\\
        &= \bbE_{\tau\sim\pi} \left[ \log p(s_0) + \sum_{t=0}^T \left (\log P(s_{t+1}|s_t,a_t) + \log \pi(a_t|s_t) \right ) \right. \\
    &~~~~~ \left. - \log p(s_0) - \sum_{t=0}^T \left (\log P(s_{t+1}|s_t,a_t) + \hr^*(s_t, a_t)\right ) \right] + \text{const}\\
    &= -\bbE_{\tau \sim p(\tau)} \left[ \sum_{t=0}^T \hr^*(s_t, a_t) - \pi(a_t|s_t)) \right] + \text{const} \\
    &= -\sum_{t=0}^T \bbE_{(s_t, a_t) \sim \rho(s_t, a_t)} [\hr^*(s_t,a_t) - \log \pi(a_t|s_t)] + \text{const}~.
\end{aligned}
\end{equation}

Without loss of generality, we approximate the finite term $\sum_{t=0}^T \bbE_{(s_t, a_t)}$ with an infinite term $\bbE_{\pi}$ by the definition, and then we have
\begin{equation}\label{eq:tau-kl-induce-2}
\begin{aligned}
    \kld(p(\tau)\|p(\tau_E))
    &\approx -\bbE_{(s, a) \sim \rho(s, a)} [\hr^*(s,a) - \log \pi(a_t|s_t)] + \text{const}\\
    &= -\bbE_{\pi} [\hr^*(s,a) - \log \pi(a|s)] + \text{const}\\
    &= -\bbE_{\pi} [\hr^*(s,a)] + \bbE_{\pi}[\log \pi(a|s)] + \text{const}\\
    &= -\bbE_{\pi} [\hr^*(s,a)] - H(\pi) + \text{const}~.
\end{aligned}
\end{equation}

Thus minimizing the KL divergence between the two trajectories is equivalent to the following optimization problem:
\begin{equation}
    \max_{\pi} \bbE_\pi\left [\hr^*(s,a)\right ]+ H(\pi)~, 
\end{equation}
which is also exactly the objective of a forward MaxEnt RL.
\end{proof}

\section{Experiments}
\label{ap:exp}
\subsection{Hyperparameters}

We show the hyperparameters for both DEEN training and policy training on different tasks in \tb{tab:hyperparameters}. Specifically, we use MLPs as the networks for training DEEN and the policy network. We note that the quality of the energy model is rather important for training the RL agent. In our implementation, we find that DEEN is sensitive to the noise scale, which should be carefully considered according to the scale of the state action data.

\begin{table*}[htbp]
\caption{Important hyperparameters used in one-dimensional experiment}
\label{tab:hyperparameters}
\centering
\resizebox{0.3\textwidth}{!}{
\begin{tabular}{llc}
\toprule
& Hyperparameter & One-D.\\
\midrule
\multirow{4}{*}{Policy} & Hidden layers & 3  \\
& Hidden Size & 200 \\
& Iterations & 6000 \\
& Batch Size & 32 \\
\midrule
\multirow{6}{*}{DEEN} & Hidden layers & 3 \\
& Hidden size & 200 \\
& Epochs & 3000 \\
& Batch Size & 32  \\
& Noise Scale $\sigma$ & 0.1 \\
& Reward Scale $\alpha$ & 1 \\
& Last Activation & \textit{tanh} \\
\bottomrule
\end{tabular}}
\end{table*}

\begin{table*}[htbp]
\caption{Important hyperparameters used in sub-optimal MuJoCo experiments}
\label{tab:hyperparameters}
\centering
\resizebox{0.55\textwidth}{!}{
\begin{tabular}{llccccc}
\toprule
& Hyperparameter & Human. & Hop. & Walk. & Swim. & Invert. \\
\midrule
\multirow{4}{*}{Policy} & Hidden layers & 3 & 3 & 3 & 3 & 3 \\
& Hidden Size & 200 & 200 & 200 & 200 & 200 \\
& Iterations & 6000 & 6000 & 6000 & 6000 & 6000 \\
& Batch Size & 32 & 32 & 32 & 32 & 32 \\
\midrule
\multirow{6}{*}{DEEN} & Hidden layers & 3 & 3 & 4 & 3 & 3 \\
& Hidden size & 200 & 200 & 200 & 200 & 200 \\
& Epochs & 3000 & 6000 & 500 & 1900  & 500 \\
& Batch Size & 32 & 32 & 32 & 32 & 32 \\
& Noise Scale $\sigma$ & 0.1 & 0.1 & 0.1 & 0.1 & 0.1 \\
& Reward Scale $\alpha$ & 1 & 5 & 1 & 1 & 1000 \\
& Last Activation & \textit{tanh} & \textit{tanh} & \textit{tanh} & \textit{tanh} & \textit{tanh}\\
\bottomrule
\end{tabular}}
\end{table*}

\begin{table*}[htbp]
\caption{Important hyperparameters used in optimal MuJoCo experiments}
\label{tab:hyperparameters}
\centering
\resizebox{0.45\textwidth}{!}{
\begin{tabular}{llcccccc}
\toprule
& Hyperparameter & Lunar. & Hop. & Walk. \\
\midrule
\multirow{4}{*}{Policy} & Hidden layers & 2 & 2 & 2 \\
& Hidden Size & 128 & 128 & 128 \\
& Batch Size & 64 & 64 & 64 \\
\midrule
\multirow{6}{*}{DEEN} & Hidden layers & 5 & 5 & 3 \\
& Hidden size & 512 & 512 & 256\\
& Epochs & 75000 & 75000 & 75000 \\
& Batch Size & 256 & 256 & 256 \\
& Noise Scale $\sigma$ & 0.05 & 0.01 & 0.1 \\
& Reward Scale $\alpha$ & 16 & 64 & 64 \\
& Last Activation & \textit{sigmoid} & \textit{sigmoid} & \textit{sigmoid}\\
\bottomrule
\end{tabular}}
\end{table*}

\subsection{Synthetic Task Training Procedure}

We demonstrate more training slices of the synthetic task in this section. 

We analyze the learned behaviors during the training procedure of the synthetic task, as illustrated by visitation heatmaps in \fig{fig:train-heat}. For each method, we choose to show four training stages from different training iterations. These figures provide more evidence that although GAIL can finally achieve good results, EBIL provides fast and stable training. By contrast, GMMIL and RED fail to achieve effective results during the whole training time.\footnote{For better understanding how these methods learn reward signals, we also visualize the changes of estimated rewards during the training procedures. Videos can be seen at \url{https://www.dropbox.com/s/0mrsoqyu040crdo/video.zip?dl=0}.}

\begin{figure*}[htb]
\centering
\includegraphics[width=0.90\linewidth]{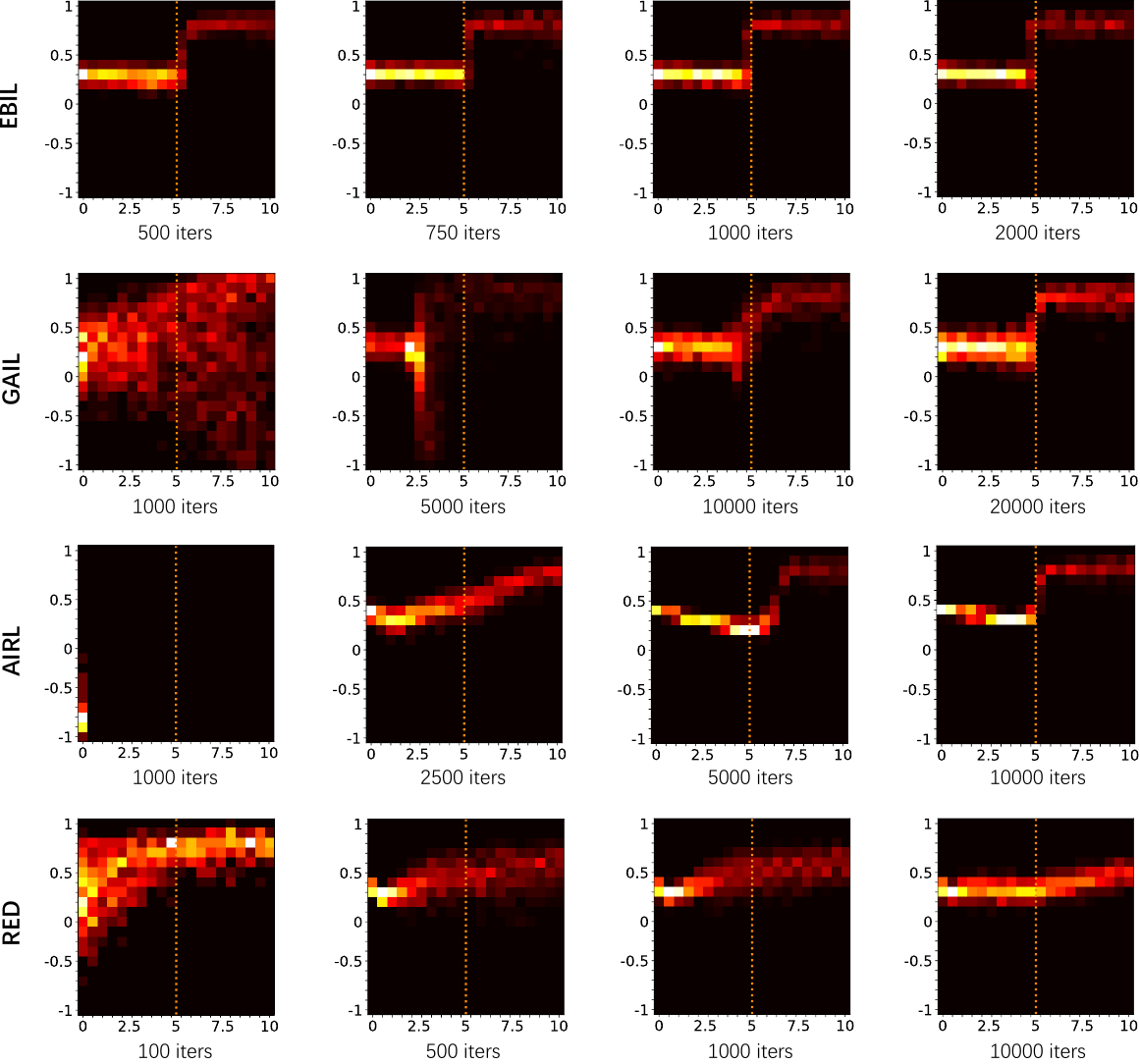}

\caption{The induced policy during policy training procedures. In each figure the horizontal axis denotes the \textit{state space}, and the vertical axis represents the \textit{action space}. Methods from top to bottom are separately EBIL, GAIL, AIRL and RED and each one contains four training stages shown in one line. The color bar is the same as \fig{fig:heat_maps}. The brighter the yellow color, the higher the visitation frequency.}
\vspace{-15pt}
\label{fig:train-heat}
\end{figure*}

\subsection{Energy Evaluation}
\label{curves-energy}

EBIL relies highly on the training of EBMs which provide the reward to imitate the expert's policy. Therefore we have to evaluate the quality of a learned EBM before training the agent. However, the loss function of DEEN is not an intuitive indicator for evaluating the learned energy network, therefore, we propose to evaluate the averaged energy value for expert trajectories and the random trajectories on different tasks. As shown in \fig{fig:curves-energy}, DEEN finally converges in all experiments by successfully differentiating the expert data from the random one. However, such a metric is a basic requirement and is always not that helpful on improving the performance of imitation learning, since the main difficulty comes from providing accurate reward signals for the trajectories that close to the experts. 


\begin{figure*}[htbp]

\centering
\vspace{-40pt}
\includegraphics[width=0.9\linewidth]{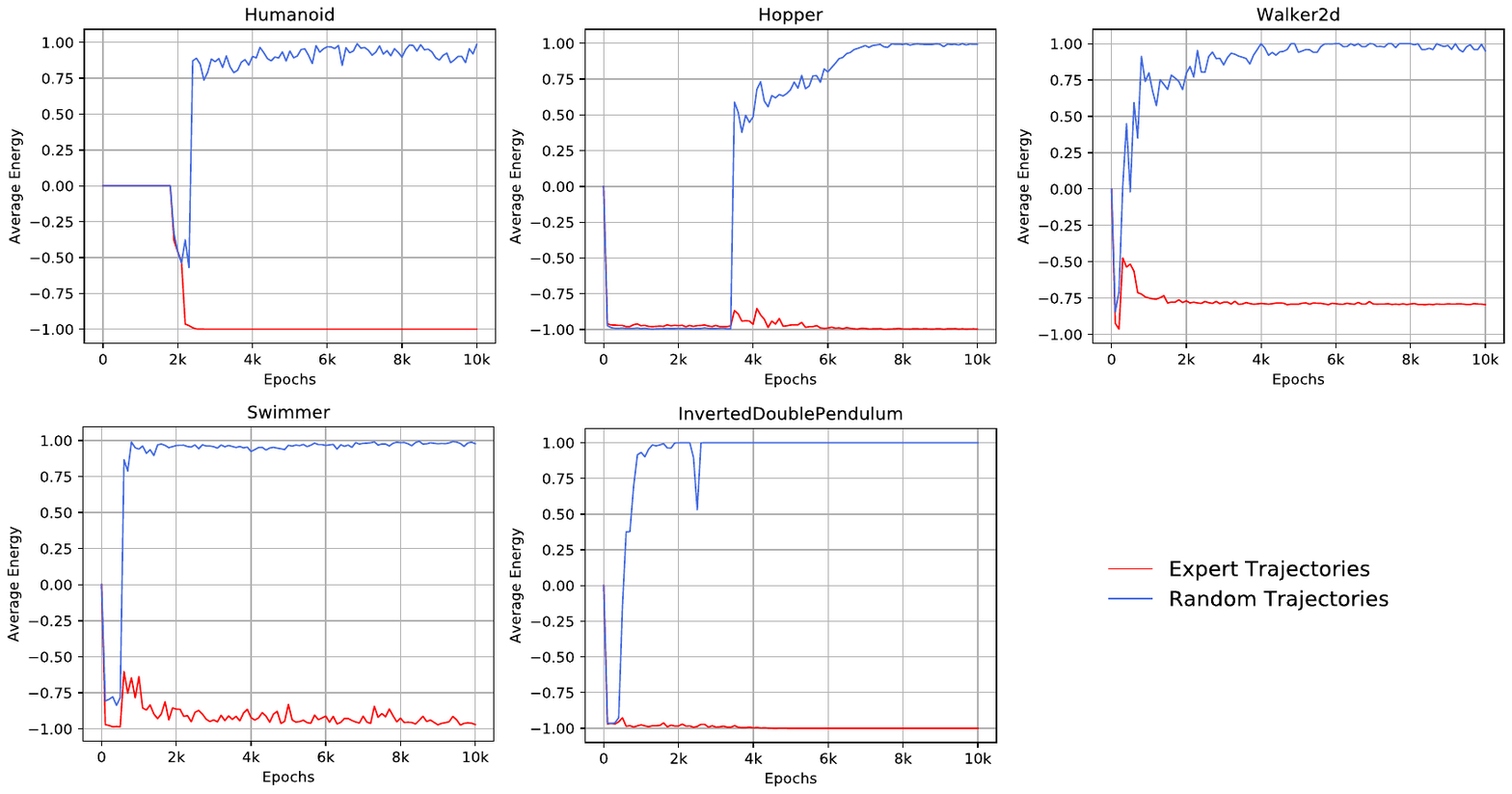}
\vspace{-60pt}
\caption{Energy evaluation curves on different MuJoCo tasks trained by sub-optimal demonstrations , where the red line represents for the average energy estimation on expert data and the blue is for random trajectories, which contain 100 trajectories separately. Note that lower energy values correspond to higher rewards. The energy value is in a range of $[-1,1]$ since we use \textit{tanh} for the last layer of the DEEN network on this task.}
\vspace{-3pt}
\label{fig:curves-energy}
\end{figure*}

\section{Further Discussions}

\subsection{Surrogate Reward Functions}

As discussed in \cite{kostrikov2018discriminator}, the reward function is highly related to the property of the task. Positive rewards may achieve better performance in the ``surviving'' style environment, and negative ones may take advantage in the ``per-step-penalty'' style environment. The different choices are common in those imitation learning works based on GAIL, which can use either $\log(D)$ or $-\log(1-D)$, where
$D\in [0,1]$ is the output of the discriminator, determined by the final ``\textit{sigmoid}'' layer. In our work, we choose ``\textit{tanh}'' but also use ``\textit{sigmoid}'' as the final layer of the energy network, which in result leads the energy into a range of $[-1,1]$ or $[0,1]$. In order to adapt to different environments while holding the good property of the energy, we can apply a monotonically increasing linear function $h$ as the surrogate reward function, which makes translation or scaling transformation on the energy outputs. It appears that in all of our tasks, the original energy signal does not show much ascendancy, and thus we choose different $h$ for these tasks.

In the one-dimension domain experiment, we choose ``\textit{tanh}'' as the final layer of the energy network, to use the following surrogate reward function:
\begin{equation}
    \hr(s,a)=h(x)=x+1~,
\end{equation}
where $\hr \in [0, 2]$ and $x=-E(s,a)$ is the energy function. Thus, the experts' state-action pair will get close-to-zero rewards at each step.

In sub-optimal MuJoCo tasks, we also choose ``\textit{tanh}'' as the final layer and construct the surrogate reward function as:
\begin{equation}
    \hr(s,a)=h(x)=(x+1)/2~.
\end{equation}
In optimal MuJoCo tasks, we choose ``\textit{sigmoid}'' as the final layer and construct the surrogate reward function as:
\begin{equation}
    \hr(s,a)=h(x)=x+1~,
\end{equation}
Note that in these tasks we make a normalized reward $\hr \in [0, 1]$ so that the non-expert's state-action pairs will gain near-zero rewards while the experts' get close-to-one rewards at each step regarding the output range of the energy is $[-1,1]$. In our experiments, similar rewards as the one-dimensional synthetic environment can also work well.

\subsection{AIRL Does Not Recover the Energy}
\label{ap:airl}

Adversarial Inverse Reinforcement Learning (AIRL)\cite{fu2017learning}, is a SoTA IRL method that apply an adversarial architecture similar as GAIL to solve the IRL problem. Formally, AIRL constructs the discriminator as 
\begin{equation}\label{eq:airl-dis1}
D(s,a) = \frac{\exp(f(s,a))}{\exp(f(s,a)) + \pi(a|s)}~.
\end{equation}
This is motivated by the former GCL-GAN work~\cite{finn2016connection}, which proposes that one can apply GAN to train GCL that formulate the discriminator as
\begin{equation}\label{eq:airl-dis2}
D(\tau) = \frac{\frac{1}{Z}\exp(c(\tau))}{\frac{1}{Z}\exp(c(\tau)) + \pi(\tau)}~,
\end{equation}
where $\tau$ denotes the trajectory. AIRL uses a surrogate reward
\begin{equation}\label{eq:airl-reward}
\begin{aligned}
r(s,a) &= \log D(s,a) - \log(1-D(s,a))\\
&= f(s,a)-\log{\pi(a|s)}~,
\end{aligned}
\end{equation}
which can be seen as an entropy-regularized reward function.

However, the difference between \eq{eq:airl-dis1} and \eq{eq:airl-dis2} indicates that AIRL does not actually recover the expert's energy since they do not separate the partition function $Z$ from $f$. Also, the learning signal that drives the agent to learn a good policy is not a pure reward term but contains an entropy term itself. We visualize the different reward choice ($f(s,a)$ or $f(s,a)-\log{\pi(a|s)}$) in \fig{fig:airlreward} as comparison, which indicates the influence of the entropy term, and verifies our intuition that AIRL in fact does not recover the expert's energy as EBIL does, but the recovered reward $f(s,a)$ can be seen as an approximation of energy.

\begin{figure*}[!t]
\centering
\subfigure[Reward as $f(s,a)-\log{\pi(a|s)}$]{
\begin{minipage}[b]{0.33\linewidth}
\label{fig:logreward}
\centering
\includegraphics[width=0.7\linewidth]{figs/heat-crop/airl_heat_40_16000_1-crop.pdf}
\end{minipage}
}
\subfigure[Reward as $f(s,a)$]{
\begin{minipage}[b]{0.393\linewidth}
\label{fig:freward}
\centering
\includegraphics[width=0.7\linewidth]{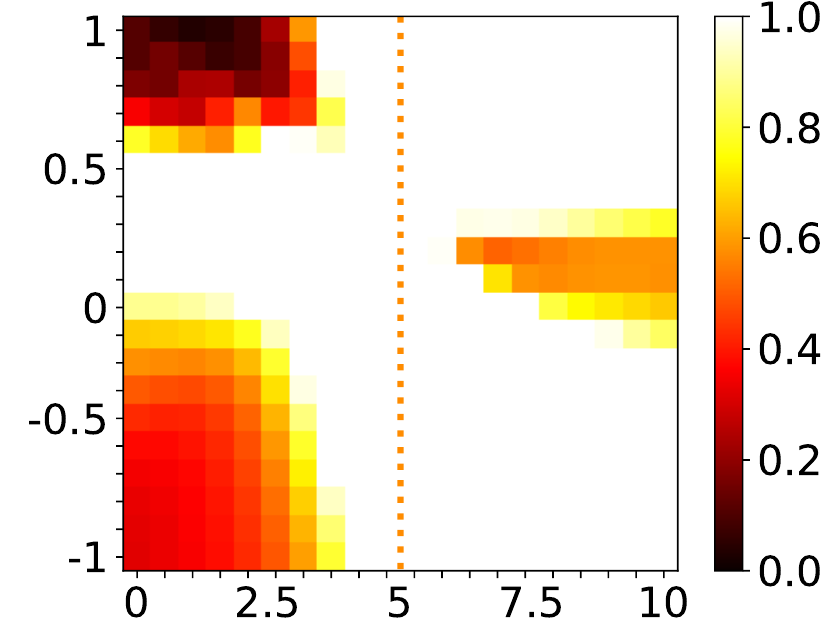}
\end{minipage}
}
\vspace{-7pt}
\caption{Heat maps of the different estimated rewards recovered by AIRL.}
\label{fig:airlreward}
\end{figure*}

\subsection{Discussions with MaxEnt RL Methods}
\label{ap:ebil-sac}

Soft-Q Learning (SQL)~\cite{haarnoja2017reinforcement} and Soft Actor-Critic (SAC)~\cite{haarnoja2018soft} are two main approaches of MaxEnt RL, particularly, they propose to use a general energy-based form policy as:
\begin{equation}\label{eq:softq}
    \pi(a_t|s_t) \propto \exp{(-E(s_t, a_t))}~.
\end{equation}
To connect the policy with soft versions of value functions and Q functions, they set the energy model $E(s_t,a_t) = -\frac{1}{\alpha}Q_{\text{soft}}(s_t,a_t)$ where $\alpha$ is the temperature parameter, such that the policy can be represented with the Q function which holds the highest probability at the action with the highest Q value, which essentially provides a soft version of the greedy policy. Thus, one can choose to optimize the soft Q function to obtain the optimal policy by minimizing the expected KL-divergence:
\begin{equation}\label{eq:softac}
    J(\pi) = \bbE_{s\sim \rho^s_{\caD}}\left[\kld\left( \pi(\cdot|s) \big \| \frac{\exp{(Q(s,\cdot))}}{Z(s)} \right) \right ]~,
\end{equation}
where $\rho^s_{\caD}$ is the distribution of previously sampled states and actions, or a replay buffer. Therefore, the second term in the KL-divergence in fact can be regarded as the target or the reference for the policy.

Consider to use the KL-divergence as the distance metric in the general objective of IL shown in \eq{eq:il}, then we get:
\begin{equation}
\begin{aligned}
    \pi^* &= \argmin_\pi \mathbb{E}_{\pi} \left [\kld \left( \pi(\cdot|s) \big \| \piE(\cdot|s)  \right) \right ]~.
\end{aligned}
\end{equation}
If we choose to model the expert policy using the energy form of \eq{eq:softq} then we get:
\begin{equation}
\begin{aligned}\label{eq:softil}
    \pi^* &= \argmin_\pi \mathbb{E}_{\pi}\left [\kld \left( \pi(\cdot|s) \big \| \frac{\exp{(-E_\piE(s,a)})}{Z} \right) \right ]~.
\end{aligned}
\end{equation}

\begin{proposition}
\label{prop:ebil-sac}
    The IL objective shown in \eq{eq:softil} is equivalent to the EBIL objective shown in \eq{eq:eb-il}.
\end{proposition}
\begin{proof}
    Since \eq{eq:eb-il} is equivalent to \eq{eq:kl-il}, it holds the optimal solution such that $\pi^*=\piE$. Also, it is easy to see that \eq{eq:softil} has the same optimal solution such that $\pi^*=\piE$.
\end{proof}
Thus, \prop{prop:ebil-sac} reveals the relation between MaxEnt RL and EBIL. Specifically, EBIL employs the energy model learned from expert demonstrations as the target policy. The difference is that MaxEnt RL methods use the Q function to play the role of the energy function, construct it as the target policy, and iteratively update the Q function and the policy, while EBIL directly utilizes the energy function to model the expert occupancy measure and constructs the target policy. 

\end{document}